%% file: final_conference.tex
\documentclass[conference]{IEEEtran}
\IEEEoverridecommandlockouts
\usepackage{cite}
\usepackage{amsmath,amssymb,amsfonts}

\usepackage{graphicx}
\usepackage{textcomp}
\usepackage{xcolor}
\usepackage{multirow}
\usepackage{adjustbox}
\usepackage{hyperref}

\usepackage[colorinlistoftodos,prependcaption]{todonotes}
\usepackage{regexpatch}
\makeatletter
\xpatchcmd{\@todo}{\setkeys{todonotes}{#1}}{\setkeys{todonotes}{inline,#1}}{}{}

\AtBeginDocument{%
  \providecommand\BibTeX{{%
    \normalfont B\kern-0.5em{\scshape i\kern-0.25em b}\kern-0.8em\TeX}}}

\usepackage{array}
\usepackage{arydshln}
\usepackage{amsthm}
\usepackage{subcaption}
\usepackage{diagbox}
\usepackage{listings}
\usepackage{float}

\usepackage{algorithm}
\usepackage{algpseudocode}

\usepackage{booktabs}

\newcommand{\vecX}{\mathbf{x}}
\newcommand{\vecY}{\mathbf{y}}
\newcommand{\vecK}{\mathbf{k}}
\newcommand{\vect}{\boldsymbol{\theta}} 
\newcommand{\vecA}{\boldsymbol{\alpha}}

\newcommand{\bbR}{\mathbb{R}}
\newcommand{\sample}{\mathcal{D}}
\newcommand{\vecT}{\boldsymbol{\Theta}}
\newtheorem{theorem}{Theorem}[section]
\newtheorem{lemma}[theorem]{Lemma}

\def\BibTeX{{\rm B\kern-.05em{\sc i\kern-.025em b}\kern-.08em
    T\kern-.1667em\lower.7ex\hbox{E}\kern-.125emX}}
\begin{document}

\title{Surrogate uncertainty estimation for your time series forecasting black-box: learn when to trust}

\author{\IEEEauthorblockN{1\textsuperscript{st} Leonid Erlygin}
\IEEEauthorblockA{
\textit{Skoltech}\\
Moscow, Russia \\
erlygin.la@phystech.edu
}
\and
\IEEEauthorblockN{2\textsuperscript{st} Vladimir Zholobov}
\IEEEauthorblockA{
\textit{Skoltech}\\
Moscow, Russia \\
zholobov.va@phystech.edu
}
\and
\IEEEauthorblockN{3\textsuperscript{rd} Valeriia Baklanova}
\IEEEauthorblockA{\textit{HSE University} \\
Moscow, Russia \\
vbaklanova@hse.ru
}
\and
\IEEEauthorblockN{4\textsuperscript{th} Evgeny Sokolovskiy}
\IEEEauthorblockA{
\textit{Sber}\\
Moscow, Russia \\
Sokolovskiy@sberbank.ru
}
\and
\IEEEauthorblockN{5\textsuperscript{th} Alexey Zaytsev}
\IEEEauthorblockA{\textit{Skoltech; BIMSA} \\
Moscow, Russia; Beijing, China \\
A.Zaytsev@skoltech.ru
}
}

\maketitle

\begin{abstract}
Machine learning models play a vital role in time series forecasting. 
These models, however, often overlook an important element: point uncertainty estimates. 
Incorporating these estimates is crucial for effective risk management, informed model selection, and decision-making.

To address this issue, our research introduces a method for uncertainty estimation. We employ a surrogate Gaussian process regression model. 
It enhances any base regression model with reasonable uncertainty estimates. This approach stands out for its computational efficiency. It only necessitates training one supplementary surrogate and avoids any data-specific assumptions. Furthermore, this method for work requires only the presence of the base model as a black box and its respective training data. 

The effectiveness of our approach is supported by experimental results. Using various time-series forecasting data, we found that our surrogate model-based technique delivers significantly more accurate confidence intervals. These techniques outperform both bootstrap-based and built-in methods in a medium-data regime. This superiority holds across a range of base model types, including a linear regression, ARIMA, gradient boosting and a neural network.
\end{abstract}


\begin{table}[!ht]

    \begin{tabular}{|l|ccc|c|}
        \hline
        \diagbox{UE method}{Base model} & OLS               & ARIMA
                                        & CatBoost          & Average
        \\
        \hline
        Base                       & 3.062             & 2.83
                                        & 3.435             & 3.109
        \\
        Bootstrap                  & 2.825             & \textbf{1.734}
                                        & 2.826             & 2.462
        \\
        Naive Surrogate                 & \underline{2.25}  & 3.064
                                        & \textbf{1.793}     & \underline{2.369}
        \\
        \textbf{Our Surrogate}          & \textbf{1.862}    &\underline{2.372}
                                        & \underline{1.946} & \textbf{2.06}
        \\
        \hline
    \end{tabular}
    \caption{Mean ranks for the Miscalibration area  for Uncertainty estimation (UE): : an UE from a base model (Base), a variant of bootstrap (Bootstrap), a naive approach to construct a surrogate UE (Naive Surrogate) and our approach to construct a surrogate UE (Our Surrogate). We get ranks for problems from TSForecasting benchmark and three types of base models: linear regression (OLS), ARIMA and
        Gradient Boosting (CatBoost) equipped with UEs. 
        A smaller rank means that a method is better, as it is close to a top one on average. The best
        results are highlighted
        with
        \textbf{{bold}}
        font, the second best results are \underline{{underscored}}.
    }
    \label{tab:teaser_table}
\end{table}

\section{Introduction}

\input{tex/introduction.tex}

\section{Related work}

\paragraph{Uncertainty types}
The uncertainty of a value is understood as its characteristic, which describes a certain allowable spread of its values and arises due to the inaccuracy of measuring instruments, the inconsistency of the allowed restrictions with the real data and the processes behind them, as well as with the approximations contained in the model itself. In machine learning, we aim at uncertainty estimation of a point prediction for a model: how confident a model is about its prediction at a particular point. There are two commonly considered types of uncertainty in machine learning: aleatoric uncertainty and epistemic uncertainty~\cite{roy2011comprehensive}.
A typical approach pays attention to both types.

Aleatoric uncertainty~\cite{abdar2021review} is related to the probabilistic nature of the data and the impossibility of overcoming it. The simplest example is an error in the data received by a measuring device with a given error. We may say that such a scatter occurs by chance and may not be eliminated. On the other hand, we may determine its characteristics using, for example, methods for building a model with inhomogeneous heteroscedastic noise.

Epistemic uncertainty~\cite{abdar2021review} is related to the limitations of the model used for forecasting. It arises due to the inaccuracy of the approximations embedded in the model or as a result of applying the model to new data that differ from those used in its construction. 
Such uncertainty may be reduced, for example, by improving the model or by using a more correct data set to train it.

Below we consider primal ideas for uncertainty estimation in regression problems. 
We start with model-agnostic ensemble methods and quantile regression.
Then we consider some model-specific approaches and argue that they can be connected within the surrogate modeling framework.


\paragraph{Ensemble approach for uncertainty quantification}

One of the approaches to obtain estimates of uncertainty is the construction of an ensemble of similar, but different in some nuances, models. At the same time, to obtain estimates of uncertainty, the spread of different model predictions is considered~\cite{lakshminarayanan2017simple,liu2019accurate}. For example, an estimate of the forecast variance at a point may be its empirical estimate from the forecast vector of an ensemble of models.

The most well-known approach of constructing an ensemble of models~--- using of bootstrap~\cite{shao1996bootstrap}. During bootstrapping, objects for training are sampled from the training set with repetitions. The resulting sample is used to train models from ensembles. With a reasonable choice of the number of models in the ensemble, this method allows one to obtain fairly accurate statistical estimates of the data.

However, such a procedure, in its basic form, considers data as a set of independent objects. Here we consider sequential data in which there is a temporal connection. There are sampling methods that extend the basic version of the bootstrap, designed specifically for working with time series and sequential data models. In particular, a block bootstrap~\cite{kunsch1989jacknife, politis1994stationary, paparoditis2001tapered} or auto-regressive data sampling~\cite{lahiri2003resampling} is used.

Recent work considers obtaining uncertainty estimates using ensembles of deep models~\cite{lakshminarayanan2017simple}. Due to the significant number of parameters available in neural networks, the slightest changes in the initialization of the model before the training lead to changes in the trained model. Moreover, estimates of the mean and scatter for each predicted point may be obtained. However, for classical machine learning models with a small number of parameters applied to a small amount of data, this approach could be better applicable since it is more difficult to achieve a variety of outputs.

\paragraph{Quantile regression.}
In quantile regression, one trains a separate model to predict the quantiles of the predictive distribution~\cite{fasiolo2021fast, salem2020prediction}.
They are preferable due to their opportunity to represent any arbitrary distribution, which varies across methods. 
Using these quantiles, we can estimate the uncertainty of a model at a point.
One of the base approaches is the usage of a pinball loss.
For example, SQR trains a separate neural network using this loss~\cite{tagasovska2019single} by trying to provide correct predictive intervals. 
The pinball loss is a tilted transformation of the absolute value function that restricts the ability to target many desirable properties (e.g., sharpness and calibration). 
As a result, SQRv2~\cite{chung2021beyond} considers different types of losses to solve that limitation presenting a current SOTA in this direction. 

\paragraph{Gaussian process approach for Uncertainty Quantification}
Sometimes there is a situation when the test set has a different distribution than the train set. This problem is known as the out-of-distribution (OOD) problem. Therefore, a model is needed to detect this and give a uniform distribution to classes.

The Gaussian Process Regression (GPR)~\cite{williams2006gaussian} model has this property. It works well even in the case of a misspecified model~\cite{zaytsev2018interpolation}. It is well-known that the GPR model was introduced as a fully probabilistic substitute for the multilayer perceptron (MLP)~\cite{neal2012bayesian}: a GPR is an MLP with infinite units in the hidden layer. Traditional GPR models have been extended to more expressive variants, such as Deep Gaussian Process~\cite{damianou2013deep}.

There are quite a lot of studies on GPR properties. For instance, it is a case where we have the misspecified problem statement~\cite{van2011information} for GPR. In~\cite{zaytsev2018interpolation}, they obtain the exact expression for interpolation error in the misspecified case for stationary Gaussian process, using an infinite-grid design of experiments. This setup is correct because it does not significantly affect the results~\cite{zaytsev2017minimax}.

In applications, GP-based models can solve various problems including classification and time-series forecasting~\cite{roberts2013gaussian}. Specifically, the article~\cite{gutjahr2012sparse} shows that GPR can be applied for time-series forecasting even for a large data scenario. In this case, a sparse variation is adopted for multiple-step-ahead forecasting.

However, there are some issues with high-dimensional problems: feature extraction is crucial for all kernel methods. To solve this problem, there was proposed a solution~\cite{NEURIPS2020_543e8374}, using spectral normalization to the weights in each layer~\cite{miyato2018spectral}. 


\paragraph{Surrogate models}
The use of an ensemble of models is justified from a theoretical point of view, but there are limitations that do not allow it to be fully applied as a universal way of estimating uncertainty. 
It is the ambiguity of solving the problem of choosing a sampling procedure for constructing an ensemble of models. Moreover, that type of method is computationally expensive: it is necessary to build a large number of models, which is not always possible for both machine learning models and less computationally efficient deep learning models.

Therefore, an alternative approach based on surrogate modelling is used. A surrogate model or meta-model $\tilde{f}(x)$ is created for a model $\hat{f}(x)$.
Due to the procedure for constructing such a model, we assume $\tilde{f}(x) \approx \hat{f}(x)$~\cite{mi2022training}.

Due to the adequacy of the quality of the model and the estimation of uncertainty, it seems natural to use regression based on Gaussian processes as a surrogate model for the original one. A similar approach was used to improve active learning~\cite{tsymbalov2019deeper}.

\paragraph{Research gap}
While there are many methods that aim at uncertainty estimation for machine learning models, there are little attention to how one can equip an already constructed machine learning model with a reasonable uncertainty estimate. 
The closest possible scenario is training of a separate head for a neural network. 
However, such an approach is model-specific.
We aim to close this research gap with introduction of a procedure to create a lightweight uncertainty estimation for a black box base model, if the training data are available.

\section{Methods}

\input{tex/methods.tex}

\section{Experiments}

\input{tex/experiments.tex}

\section{Conclusions}

The problem of efficient uncertainty estimation for a black box model is important in many applications.
We propose a surrogate uncertainty estimate.
The computational complexity of the developed method coincides with that of Gaussian process regression for the available training sample, making additional computational efforts small for most models.

Our surrogate uncertainty estimation produces accurate confidence interval predictions for different base models on different datasets. 
Calibration and regression metrics of our surrogate model are comparable with classical bootstrap ensemble methods and in average are better. 

We anticipate that the presented work will guide practitioners in estimating uncertainty using the proposed pipeline featuring a Gaussian process regression surrogate. 
It is intuitive, easy to use, and provides good results in various scenarios.

Further improvement of this model can be an application of it in NLP-related problems~\cite{shelmanov2021certain}, increasing efficiency via the usage of efficient Gaussian process regression and deep kernel learning~\cite{wilson2016deep,burnaev2015surrogate,ober2021promises}.
We also expect that a similar approach should work for a wider set of regression and classification problems. 

\section*{Acknowledgment}
The work by Leonid Erlygin was supported by a grant for research centers in the field of artificial intelligence, provided by the Analytical Center for the Government of the Russian Federation in accordance with the subsidy agreement (agreement identifier 000000D730321P5Q0002 ) and the agreement with the Ivannikov Institute for System Programming of the Russian Academy of Sciences dated November 2, 2021 No. 70-2021-00142.
The work of others was supported by Sber.


\clearpage
\bibliographystyle{IEEEtran}
\bibliography{main}

\clearpage
\appendix

\input{tex/appendix.tex}
\end{document}

%% file: tex/introduction.tex
\begin{figure*}[htpb]
  \centering
 \subcaptionbox{}{%
    \includegraphics[width=1.2\columnwidth]{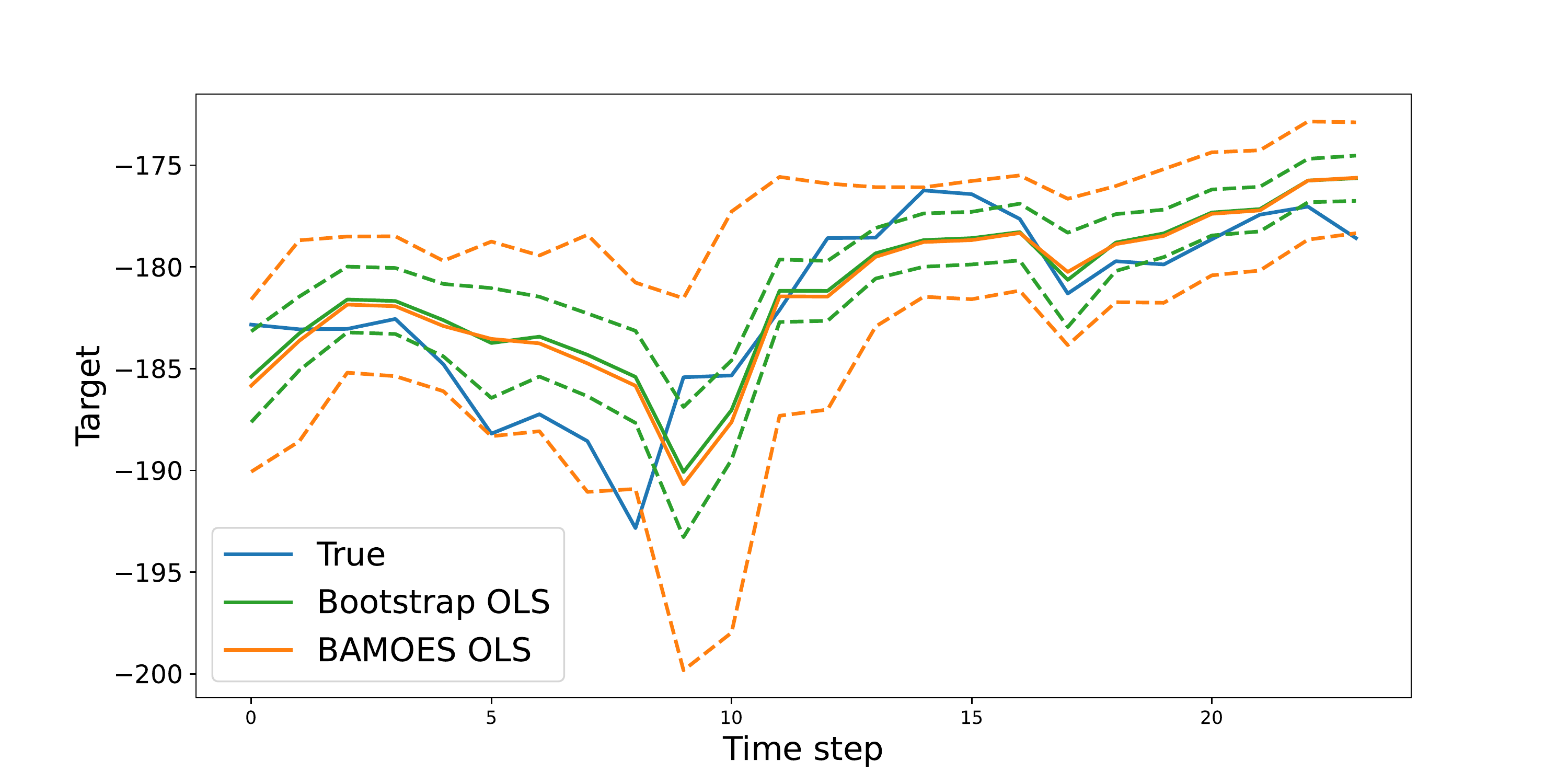}%
  }
  \subcaptionbox{}{%
    \includegraphics[width=0.8\columnwidth]{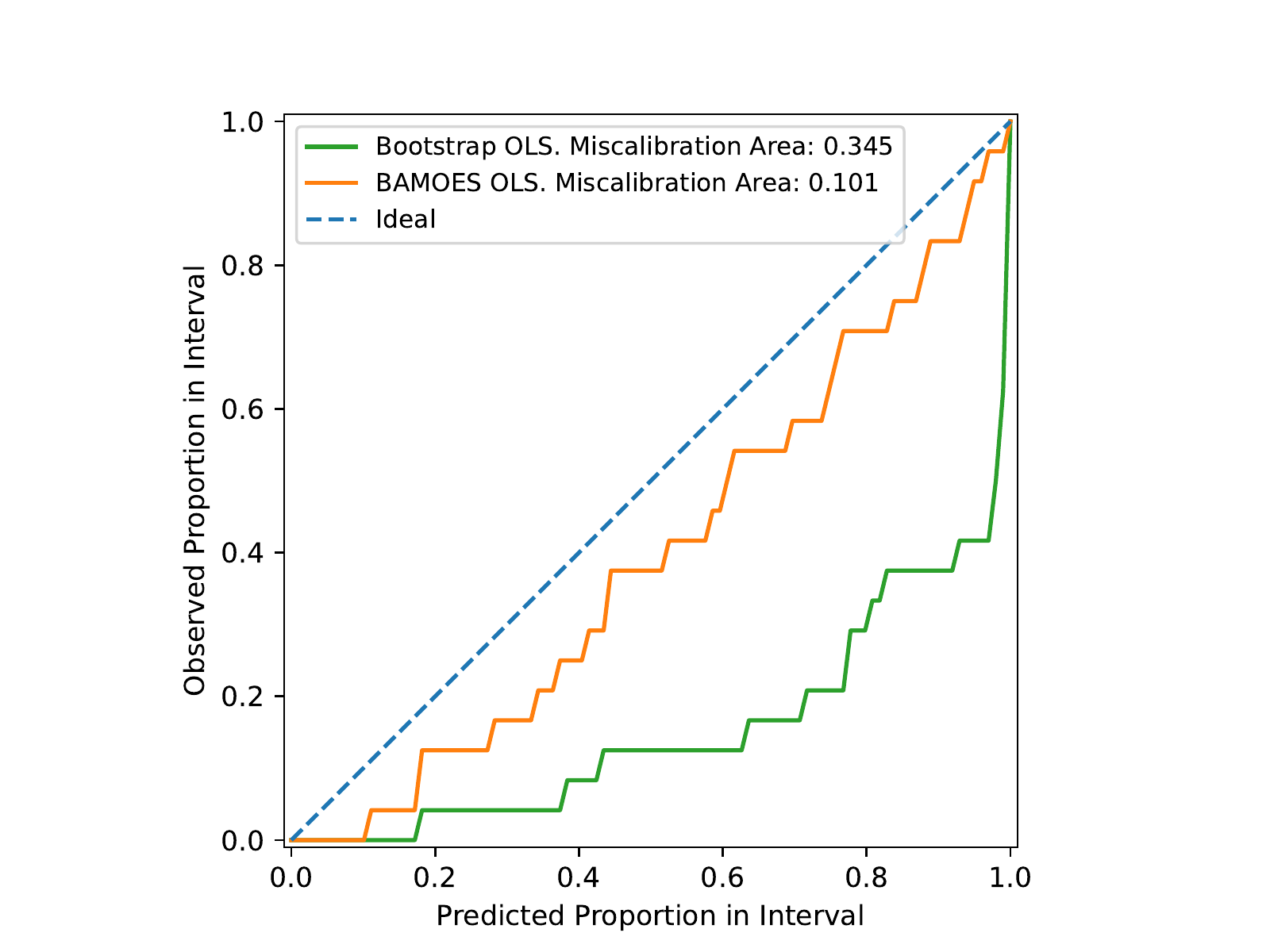}%
  }

  \caption{Example of uncertainty estimation with our method and other methods: the left plot shows obtained predictions and corresponding uncertainty estimates, and the right plot provides insight into the quality of the uncertainty estimation. The description of dataset A is available below.
  (a)
  Uncertainty estimation for our method BAMOES and a Bootstrap method for one-step-ahead time series forecasting problem. Solid lines are model predictions, and dashed lines are $0.95$ confidence intervals. Our approach is more adequate, especially during the sudden change of the true values of the target function. It isn't overconfident and better reflects anomaly near the $10$-th time step.
  (b) Comparison of true and estimated quantiles for the left plot. Our BAMOES method provides better results, being close to the blue dashed diagonal that corresponds to perfect calibration. We see them from curves themselves and from miscalibration areas that we want to minimize.
  }
  \label{fig:teaser_figures}
\end{figure*}

The ability to estimate the uncertainty of the model's predictions is a long-standing problem in machine learning, which has high practical value. It is particularly important in risk-sensitive areas such as medical machine learning~\cite{Kompa2021} and object detection in autonomous driving~\cite{9525313}. 

Among other applications, uncertainty estimation is desirable in time series forecasting.
The problem is crucial for both classic~\cite{lahiri2003resampling} and deep learning models as well~\cite{lim2021time,zhu2017deep}.
For application examples, you can look at share price prediction tasks \cite{2210.17030}:
it is important to detect when macroeconomic change occurs and new data start to come from a different distribution. Accurately estimated confidence intervals could help to detect such changes: confidence intervals are wide on out-of-distribution data. 

We consider uncertainty estimation for a fixed pre-trained deterministic regression model, which we will call the \emph{base model} throughout the paper.
The base model accurately predicts target value but lacks an uncertainty estimation for its predictions.
We want to equip it with the ability to produce uncertainty estimates for predictions and corresponding confidence intervals. 
This estimate should add little to the computational expenses for inference of the base model. 

The most natural model-agnostic approach is a bootstrap.
However, it faces two challenges: dependencies in sequential time series and a requirement to train multiple models and to run all of them during inference.
The former problem can be mitigated by taking the dependency structure into account, while you need to make correct assumptions about the dependence structure~\cite{lahiri2003resampling}.
The latter problem can be mitigated with e.g., MCMC dropout approach~\cite{lakshminarayanan2017simple}, reducing the quality of uncertainty estimation~\cite{shelmanov2021certain}. 
Other alternatives require other hard-to-fulfill assumptions and model-specific methods. 

The alternative proposed in this paper is to train a \textit{surrogate} model on the same input data used to train the base model.
We select Gaussian process regression as a functional class for a surrogate model, as it provides reliable uncertainty estimates~\cite{koziel2013surrogate} and is straightforward to train~\cite{williams2006gaussian}.
If a problem requires representation learning, we can train a surrogate model that takes embeddings as an input or in other way approach deep kernel learning~\cite{wilson2016deep}.
A similar approach was used to provide uncertainty estimates for image classification problems \cite{NEURIPS2020_543e8374}. Their Gaussian process classifier was trained on top of hidden image representations computed by CNN. 

After surrogate model training, we construct a combined model, which uses a base model to make target value prediction and a surrogate model to estimate uncertainty for the prediction of the base model. 
To make uncertainty estimation reliable, we design a loss function that allows the surrogate model to imitate the base model predictions, while keeping the method computationally efficient and avoiding contamination of the training sample with points we are uncertain about. 

We compare the proposed method for uncertainty estimation with bootstrap ensemble methods ~\cite{shao1996bootstrap}, Gaussian Process Regression (GPR)~\cite{williams2006gaussian} and Quantile Regression (QR)~\cite{fasiolo2021fast, salem2020prediction}. 
Our evidence includes experiments with different base models: linear models, ARIMAs~\cite{box2015time}, and gradient boosting~\cite{dorogush2018catboost}.
Most of these methods have their own intrinsic ways of uncertainty estimation. See \cite{snyder2001prediction} for ARIMA and  \cite{malinin2020uncertainty,duan2020ngboost} for gradient boosting. 
The experiments show that produced uncertainty estimates with our surrogate approach are more accurate than predictions from even built-in methods designed specifically for these models. 
In this paper, we focus on the time-series forecasting problem as one of the most challenging, as it requires close attention to the structure of the dependence of the data with a convenient tool for the comparison of performance over diverse regression problems.
To enrich the comparison, we thoroughly compare our approach with a wide range of methods based on time-series-specific bootstraps.

An example of our uncertainty estimation in Figure~\ref{fig:teaser_figures} demonstrates that the surrogate uncertainty estimates don't suffer from overconfidence typical for other methods.
Moreover, Table~\ref{tab:teaser_table} shows that this evidence is not anecdotal: it holds for a wide range of datasets and types of base models: our approach to the construction of a surrogate model outperforms basic uncertainty estimates for considered classes of models the best bootstrap-based approach we found and naive training of surrogate uncertainty estimation.

To sum up, our study resulted in the following key findings:
\begin{itemize}
    \item We propose a surrogate uncertainty estimation BAMOES for a black box model. The implementation requires only an existing model functioning as a black box and a training sample.
    The associated loss function is constructed to constrain surrogate model to have similar predictions to black box model, while ensuring that surrogate model accurately models prediction uncertainty.
    The proposed loss function can be efficiently calculated.
    To make the base model and a surrogate model closer, we propose to use a specific design of experiments during the surrogate model training.
    \item The added computational costs during the training and inference stages are minimal. We prove this specifically for the Gaussian process regression surrogate utilized in the study.
    \item We created a benchmark based on TSForecasting, a collection of datasets designed for testing time series forecasting methodologies. TSForecasting encompasses diverse datasets from various domains, each with unique characteristics.
    \item The quality of uncertainty estimates produced by our method surpasses that of model-specific approaches and time-series-specific bootstrap methods. 
    This finding is consistent across different classes of base black box models, including linear autoregression, gradient boosting, and neural networks. 
    A comparable approach can also be adapted to scenarios where only a training sample or a black box model is available, albeit with a decrease in performance.
\end{itemize}

We conduct experiments for the time series forecasting problem, while we don't use any specific properties of this problem. It is likely, that similar results hold for a more general class of regression problems, while we don't provide evidence in this paper.


%% file: tex/methods.tex
\subsection{Problem formulation}
\label{sec:problem}

Let us have a training dataset $\mathcal{D} = \{(\vecX_i, y_i)\}_{i = 1}^N$, where $\vecX_i$ is an input data sample from the domain $\mathcal{X} \subseteq \bbR^d$, and $y_i$ is a corresponding target from the domain $\mathcal{Y} \subseteq \bbR$. 
The training dataset is sampled from a joint distribution of inputs and targets 
$p(\vecX, y)$. 
We also make standard Bayesian assumptions for the data generation process: firstly, parameters $\vect \in \vecT$ of a function $f_{\vect}: \mathcal{X} \rightarrow \mathcal{Y}$ are sampled, then $y$ is sampled from the conditional distribution $p(y | \vecX, \vect)$. 
For a regression problem, the conditional distribution can be Gaussian: $p(y | \vecX, \vect) = \mathcal{N}(y; f_{\vect} (\vecX), \sigma^2_{\vect} (\vecX))$. 

Given these assumptions, the predictive distribution can be written as:
\begin{equation}
    p(y | \vecX, \sample) = \int_{\vect} p(y | \vecX, \vect) p(\vect | \sample) d \vect,
\end{equation}
where $p(y | \vecX, \vect)$ is the target density given parameters $\vect$ and a point $\vecX$, $p(\vect | \sample)$ is the posterior distribution of parameters $\vect$. 

In practice, the predictive distribution is rarely tractable, so many methods use point estimates: $ p(y | \vecX, \sample) \approx p(y | \vecX, \vect)$, where $\vect$ can be e.g., a maximum likelihood estimate, substituting $p(\vect | \sample)$ with a delta-function.
In this case, we treat the mean value of this distribution $\hat{f}(\vecX)$ as the prediction of the model.

We can formulate our problem in two ways.
The first goal is an accurate estimation of the variance $\hat{\sigma}^2(\vecX)$ of the distribution $p(y | \vecX, \sample)$ at a point $\vecX$ quantifying the uncertainty about the model prediction.
The second goal is the estimation of the shortest confidence interval of the significance level $\alpha$ such that the true value fall into this interval with probability $\alpha$, and the interval is the shortest among all such intervals.
If $p(y | \vecX, \vect)$ is Gaussian, two formulations coincide, as the confidence interval for the prediction with probability $\alpha$ can be written as 
\begin{equation}
\label{eq:ci}
 \hat{\mathrm{CI}}_{\alpha}(\vecX) = [\hat{f}(\vecX) - z_{\alpha / 2} \hat{\sigma}(\vecX), \hat{f}(\vecX) + z_{\alpha / 2} \hat{\sigma}(\vecX)],  
\end{equation}
where $z_{\alpha / 2}$ is the $\alpha / 2$ quantile of the standard normal distribution.

If we have a probabilistic model, these problems admit reasonable solutions.
However, in many cases, we have only a black-box deterministic \emph{base} model $\hat{f}(\vecX)$.
So, our final goal is to equip a deterministic regression model with accurate uncertainty estimation.

Below we propose a surrogate modeling approach, which can be used to estimate the uncertainty of a deterministic base model.

\subsection{Naive surrogate based on Gaussian process regression}
\label{surrogate}

We present a universal and numerically efficient approach that requires no assumptions about a model and can equip any black-box model with an uncertainty estimate.
Given a deterministic \emph{base} model $\hat{f}(\vecX)$, we introduce another \emph{surrogate} model $s(\vecX)$.
We select $s(\vecX)$ such that it directly models the target distribution $p(y | \vecX, \sample)$ while mimicking the black-box model predictions from $\hat{f}(\vecX)$.

A natural choice for such a probabilistic model is the Gaussian process regression (GPR)~\cite{williams2006gaussian} as it has superior performance on several tasks  that requires uncertainty estimation such as Bayesian optimization and active learning~\cite{WANG2014167}.  


Intuitively, the surrogate model trained on the same dataset approximates base model $s \approx \hat{f}$, which itself approximates underlying distribution $\hat{f}(\vecX) \approx \mathbb{E}_{ p(y|\vecX, \sample)} y$.
Thus, we expect it to have adequate uncertainty estimates.
For GPR, the uncertainty estimate will have a natural kernel-style behavior: the uncertainty increases as we go away from the points from the initial sample used to train the base model.


\begin{figure}[ht]
    \centering
    \includegraphics[width=0.45\textwidth]{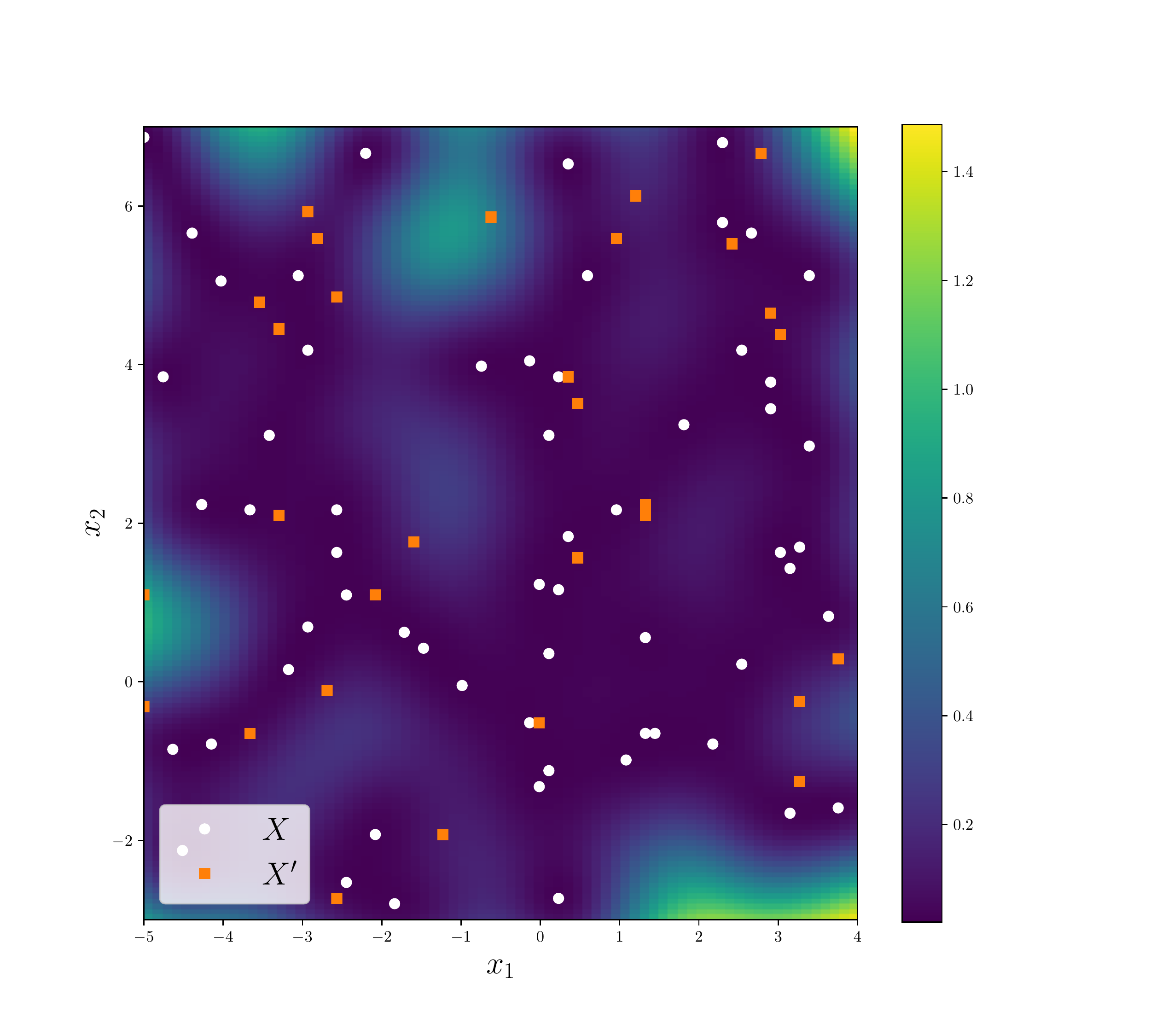}
    \caption{Example of Uncertainty estimation for two-dimensional input via a matching surrogate model. The variance estimate corresponds to the fill color at the point. At points from the initial training sample $X$, the uncertainty is almost zero, while for points from the additional sample $X'$, it takes reasonable values reflecting our absence of knowledge about the true function values at these locations.}
    \label{fig:uq_example_2d}
\end{figure}

\subsection{Base-model-enhanced surrogate model (BAMOES)}

Using the approach described above as it is to obtain a surrogate model is naive, as we expect that if the training sample is the same, then the trained model would be the same in cases of the base model and a surrogate model.
Numerous pieces of evidence suggest that this is not true, and we can get different models, even if we use the same dataset and the same class of models.

On the other hand, the base model is available as a black box, so one can query it at the points of interest, improving the surrogate model by showing it more relevant training data.
For Gaussian process regression, we can use this information efficiently.

We assume that the surrogate model $s(\vecX)$ is a realization of the Gaussian process.
More precisely, $s(\vecX) \sim GP(\mu, k(\vecX, \vecX') | \mathcal{D})$ for a covariance function $k(\vecX, \vecX')$ from a parametric family, a learnable $\mu$. The realizations are conditioned on the available data $\sample$. 
The conditional mean $m(\vecX)$ and variance $\sigma^2(\vecX)$ given the sample of observations $\mathcal{D}$ at a new point $\vecX$ in this case have the following form:
\begin{align*}
    m(\vecX) &= \vecK^T K^{-1} \vecY = \vecK^T \vecA, \vecA = K^{-1} \vecY, \\
    \sigma^2(\vecX) &= k(\vecX, \vecX) - \vecK^T K^{-1} \vecK,
\end{align*}
where $\vecK = \{k(\vecX, \vecX_i)\}_{i = 1}^N$, $K = \{k(\vecX_i, \vecX_j)\}_{i,j = 1}^N$.
To find the parameters of the covariance function, one maximizes the likelihood of the data given the covariance function.

In our case, we have an additional requirement to match the base model $\hat{f}(\vecX)$.
So, the loss function is the sum of two terms:
\begin{align}
\label{eq:gp_loss}
L(s, \hat{f}, \sample, \sample^{\hat{f}}) = 
& -(1 - C) \log p(\vecY| X, s) + \\
& + C \sum_{i = 1}^{L} (s(\vecX_i') - \hat{f}(\vecX_i'))^2,  \nonumber   
\end{align}
where $\log p(\vecY| X, s)$ is the data log-likelihood for the parameters of $s$ and $C \in [0, 1]$ is the weight coefficient.
The second term is the sum of the squared difference between the surrogate model prediction $s(\vecX_i')$ and the base model $\hat{f}(\vecX_i')$ for a sample $\sample^{\hat{f}} = (X', \mathbf{\hat{f}}') = \{(\vecX'_i, \hat{f}(\vecX_i'))\}_{i = 1}^L$.
In our experiments, inputs in $\sample^{\hat{f}}$ are selected uniformly randomly over the domain of interest.
So, to calculate the loss, we need only a training sample and a base model available as a black box.

In formulated optimization objective additional points from $\sample^{\hat{f}}$ are used to adjust the surrogate model parameters to better match prediction of base model.
We call surrogate model trained with this objective \textit{Base-model-enhanced surrogate model} (BAMOES).

In the terminology of sparse Gaussian process regression, points from $\mathcal{D}$ are \emph{inducing} points~\cite{quinonero2005unifying}, as we condition our distribution on them.
In this way, we (a) keep the computational complexity low and (b) uncertainty estimation growing if we go away from the initial training sample.

Let us put these two statements formally.
\begin{lemma}
The computational complexity for the evaluation of the loss function~\eqref{eq:gp_loss} equals $O(N^3) + O(L N)$.
\end{lemma}

\begin{proof}
Using the formula for the likelihood of Gaussian process regression, we get 
\begin{align*}
&L(s, \hat{f}, \mathcal{D}, \mathcal{D}^{\hat{f}}) = \frac{(1 - C)}{2} \left(N \log \pi + \log \det |K| + \vecY^T \vecA \right) + \\
&+ C \left(K_{X'\, X} \vecA - \hat{f}(X') \right)^T \left(K_{X'\, X} \vecA - \hat{f}(X') \right),    
\end{align*}
where $\hat{f}(X') = \{ \hat{f}(\vecX_1'), \ldots, \hat{f}(\vecX_{L}')\}$ and $K_{X'\, X} = \{k(\vecX'_i, \vecX_j)\}_{i, j = 1}^{L, N}$.

So, we need to evaluate two terms: the likelihood and the squared loss.
To calculate the likelihood, we need $O(N^3)$, as we need the inverse and the determinant of the covariance matrix of size $N \times N$.
Note that if we have the inverse, we calculate $\boldsymbol{\alpha} = K^{-1} \vecY$ in $O(N^2)$.
So, to get the predictions, we need $O(L N)$ additional operations in addition.
Summing both complexities, we obtain the desired $O(N^3) + O(L N)$.
\end{proof}
So, as long as we keep $L$ of an order of magnitude similar to $N$, we have little additional computational power required.
Moreover, we can afford $L$ to be of order $N^2$, which is impossible with the naive baseline above.

The second statement about the behavior of uncertainty as we go away from a training point is natural. 
In points from $X$, a standard property for Gaussian process regression holds and the uncertainty is equal to the noise variance in the data.
Typically, the estimation of noise variance is small.
If we move $\vecX$ to infinity, then the uncertainty estimates $\sigma^2(\vecX)$ is $k(\vecX, \vecX)$ for any reasonable covariance function, whatever $\sample^{\hat{f}}$ were used, as the components of the covariance vector $k(\vecX, \vecX')$ goes to zero.
For intermediate points like those from $X'$ the estimate uncertainty is higher than for the points from $X$.
This behaviour in general corresponds to our idea on how uncertainty for $\hat{f}(\vecX)$ should look like.
An example of the application of our approach is presented in Figure~\ref{fig:uq_example_2d}.

\paragraph{Surrogate-model-aware inference}
After the surrogate model is trained, we use the following combined model: the point target values predictions come from base model $\hat{f}(\vecX)$, and the variance $\hat{\sigma}^2(\vecX) = \sigma^2(\vecX)$ from the surrogate model.
We assume that the distribution of the output is Gaussian and can use the formula~\eqref{eq:ci} to produce confidence intervals if required.
The added computational complexity of our approach during inference is the evaluation cost for the surrogate model variance.

%% file: tex/experiments.tex

We structure the experiments section in the following way. 
Subsection~\ref{sec:dataset} presents the used time series forecasting benchmark.
Subsection~\ref{sec:metrics} introduced used quality metrics.
Then we describe the main comparison of our approach with others for different types of base models.
The subsection on ablation study concludes this section.
In all tables below best results are highlighted with \textbf{{bold}} font, second best results are \underline{{underscored}}.

\subsection{Datasets}
\label{sec:dataset}

\paragraph{Time series forecasting data FD benchmark}
One of the largest benchmarks for predicting one-dimensional time series is the Monash time series forecasting archive (TSForecasting)~\cite{godahewa2021monash}. 
It contains $26$ publicly available time series datasets from different applied domains with equal and variable lengths. 
The goal for each dataset is time series forecasting for a specific time horizon $h$. 
The data cover nine diverse areas: tourism, banking, Internet, energy, transport, economy, sales, health, and nature.

Some sets repeat in slightly different versions: the frequency of the time series considered in them changes (day, month, quarter, or year), or missing values are included and excluded. Because of this, the total number of sample options reaches $50$. 
    
Our choice of TSForecasting is motivated by its diversity, enabling us to tackle a broad range of tasks. Typically, a single sample encompasses an extensive array of time series, averaging about $2600$ per sample. Additionally, the dataset includes six extraordinarily long time series, two of which exceed a length of seven million. 

Original paper~\cite{godahewa2021monash} provides a detailed description of each dataset. All the metadata for building the model (prediction horizon, context length, periodicity, etc.) follow the settings from this project. For some datasets in~\cite{godahewa2021monash}, there was no metadata. We excluded such datasets from consideration. Therefore, the number of datasets has decreased to $19$ datasets with one-dimensional data and one dataset with single multi-dimensional time series. From each dataset with one-dimensional data, the first two time series were taken. 

For one-dimensional time series from TSForecasting with target values only, we use $k$ lags as features. For example, $i$-th point has label $y$ and features $\vecX$:
\[
    y = y_i, \quad \vecX = (y_{i - 1}, \ldots, y_{i - k}).
\]
For multi-dimensional time series, we use provided features.

We split time series into train and test data using test data with the size of $(h + k)$ for one-dimensional time series and $h$ for multi-dimensional time series. Moreover, to speed up computations, only time series of length $\max(2 \cdot lag, 200)$ are used. 
Overall we test on $30$ datasets with $2$ time series taken from each of them.

For additional experiments we consider a dataset for the time series prediction connected to the financial industry.
It has a clear out-of-distribution parts related to the changes caused by COVID-19 and other crises, so the uncertainty estimation in this case is desirable.
Moreover, a subsidiary challenge comes from the small number of points available for training.



\subsection{Evaluation details}
\label{sec:metrics}

To provide a multi-faceted evaluation, the results include values of various quality metrics for uncertainty quantification. 
We present values of the Root Mean Square Calibration Error (RMSCE)\cite{chung2021uncertainty}, miscalibration area \cite{tran2020methods} and Expected Normalized Calibration Error (ENCE)\cite{ence} in the main text.
Alongside with calibration metrics, we also compute regression metric, Root Mean Squared Error (RMSE).

We use critical difference (CD) diagrams~\cite{demvsar2006statistical} to compare methods for a selected metric via implementation from~\cite{IsmailFawaz2018deep}. 
The vertical lines in the diagram correspond to the mean rank of a method over a range of considered datasets or problems.
A thick horizontal line unites a group of models that are not significantly different in terms of the metric value.

\subsection{Considered uncertainty estimates}
\label{sec:considered_models}

The complexity of the uncertainty estimation problem leads to diverse solutions specific to different approaches.
Our question is how we can improve the intrinsic and model-agnostic methods for different classes of base models.
We consider ordinary least squares (OLS),
CatBoost (a realization of Gradient boosting equipped with uncertainty estimate), ARIMA, and QR approaches - using pinball loss (SQR PL) and interval score (SQR CL)~\cite{chung2021beyond}.
For all methods, the hyperparameters are default.

\begin{table}[H]
\scriptsize
\caption{Ranks of calibration metrics of all best models aggregated over Forecasting data.
Here our surrogate models use OLS as base model. Best results are highlighted with \textbf{{bold}} font, second best results are \underline{{underscored}}
}
\label{tab:rank_all_best}
\begin{adjustbox}{width=0.5\textwidth}
\begin{tabular}{ccccc}
\toprule
\begin{tabular}{c}
Model
\end{tabular}&
\begin{tabular}{c}
RMSE $\downarrow$
\end{tabular}&
\begin{tabular}{c}
Miscal.\\
Area $\downarrow$
\end{tabular}&
\begin{tabular}{c}
RMSCE $\downarrow$
\end{tabular}&
\begin{tabular}{c}
ENCE $\downarrow$
\end{tabular}\\
\midrule
OLS & \underline{14.987} &  27.474 &  27.408 &  31.289 \\
BSAP BS OLS & \textbf{14.658} &  23.724 &  23.974 &  32.329 \\
Naive BS OLS &  16.724 &  26.671 &  26.645 &  29.658 \\
SQR PL &  17.658 &  20.908 &  20.882 &  19.763 \\
SQR CL &  17.75 &  18.132 &  18.224 &  21.329 \\
ARIMA &  35.276 &  31.513 &  31.539 &  23.961 \\
BSAP BS ARIMA &  17.75 &  19.408 &  19.276 &  21.026 \\
Naive BS ARIMA &  18.303 &  23.066 &  23.0 &  26.684 \\
Naive surrogate & \underline{14.987} & \underline{13.513} & \underline{13.039} & \textbf{13.0} \\
GPR &  16.882 &  15.75 &  15.605 &  14.053 \\
BAMOES & \underline{14.987} &  13.816 &  13.658 & \underline{13.053} \\
BAMOES RBF kernel & \underline{14.987} & \textbf{12.829} & \textbf{12.724} &  14.237 \\
\bottomrule
\end{tabular}
\end{adjustbox}
\end{table}

For all base models, we consider different ways to construct uncertainty estimates. 
We start with built-in approaches for each base model.
OLS and ARIMA models obtain uncertainty estimates based on a Bayesian assumption about the model parameters. 
The gradient boosting model uses auxiliary models that minimize the quantile loss.

We compare the built-in approaches with alternatives that use the base model as a black box or can train it in case of bootstraps.
Considered bootstraps include Naive bootstrap (Naive BS) and advanced bootstrap types Maximum Entropy-based bootstrap (MEB),
Stationary Block Bootstrap (SBB), and
Bootstrapping Stationary Autoregressive processes (BSAP BS).
To highlight the benefits of our approach, we provide another \emph{Naive} surrogate baseline that was trained without additional knowledge acquired from a base black box model.
In the ablation study, we also compare our approach with alternative surrogate-based ones.


\begin{table}
\scriptsize
\caption{Ranks of regression and uncertainty estimation metrics aggregated over Forecasting data benchmark}
\label{tab:rank_table}
\begin{adjustbox}{width=0.5\textwidth}
\begin{tabular}{lccccc}
\toprule
\begin{tabular}{c}
Uncertainty \\ estimate
\end{tabular}&
\begin{tabular}{c}
Base \\ model
\end{tabular}&
\begin{tabular}{c}
RMSE $\downarrow$
\end{tabular}&
\begin{tabular}{c}
Miscal.\\
Area $\downarrow$
\end{tabular}&
\begin{tabular}{c}
RMSCE $\downarrow$
\end{tabular}&
\begin{tabular}{c}
ENCE $\downarrow$
\end{tabular}\\
\midrule
Built-in & \multirow{7}{*}{OLS} & \underline{3.825} &  4.075 &  4.075 &  4.075 \\
Naive BS & &  4.438 &  3.988 &  4.112 &  3.975 \\
MEB BS & &  4.95 &  4.912 &  4.95 &  5.025 \\
SBB BS & &  4.088 &  5.412 &  5.312 &  5.488 \\
BSAP BS & & \textbf{3.05} &  4.075 &  4.1 &  5.012 \\
Naive surrogate & & \underline{3.825} & \underline{2.962} & \underline{2.862} & \underline{2.438} \\
BAMOES (ours) & & \underline{3.825} & \textbf{2.575} & \textbf{2.588} & \textbf{1.988} \\
\hline
Built-in & \multirow{7}{*}{ARIMA} &  4.957 &  4.717 &  4.739 &  4.087 \\
Naive BS & & \textbf{2.0} & \underline{3.054} & \underline{3.022} & \textbf{3.489} \\
MEB BS & &  3.804 &  4.043 &  4.076 &  3.793 \\
SBB BS & &  4.87 &  4.413 &  4.457 &  3.848 \\
BSAP BS & & \underline{2.196} & \textbf{2.728} & \textbf{2.728} & \underline{3.522} \\
Naive surrogate & &  5.087 &  5.098 &  5.054 &  5.522 \\
BAMOES (ours) & &  5.087 &  3.946 &  3.924 &  3.739 \\
\hline
Built-in & \multirow{7}{*}{CatBoost} & \underline{3.978} &  5.178 &  5.133 &  3.6 \\
Naive BS & &  3.989 &  4.344 &  4.389 &  4.844 \\
MEB BS & &  4.733 &  4.5 &  4.489 &  4.889 \\
SBB BS & &  4.2 &  5.411 &  5.389 &  6.311 \\
BSAP BS & & \textbf{3.144} &  3.811 &  3.911 &  4.378 \\
Naive surrogate & & \underline{3.978} & \textbf{2.256} & \textbf{2.178} & \underline{2.0} \\
BAMOES (ours) & & \underline{3.978} & \underline{2.5} & \underline{2.511} & \textbf{1.978} \\
\bottomrule
\end{tabular}
\end{adjustbox}
\end{table}

\begin{table}
\scriptsize
\caption{Calibration metrics for a base neural network model. We aggregate them and calculate ranks for the first 50 time series in the yearly subset of the M4 forecasting dataset~\cite{m4_dataset}.
}
\label{tab:m4_dataset}
\begin{adjustbox}{width=0.5\textwidth}
\begin{tabular}{lcccc}
\toprule
\begin{tabular}{c}
Model
\end{tabular}&
\begin{tabular}{c}
Mean\\
Miscalibration Area $\downarrow$
\end{tabular}&
\begin{tabular}{c}
Mean\\
ENCE $\downarrow$
\end{tabular}&
\begin{tabular}{c}
Rank\\
Miscalibration Area $\downarrow$
\end{tabular}&
\begin{tabular}{c}
Rank\\
ENCE $\downarrow$
\end{tabular}\\
\midrule
MC dropout &  0.403 &  4.083 &  4.16 &  4.18 \\
Naive surrogate & \underline{0.249} &  2.58 & \textbf{2.52} &  2.9 \\
Deep Naive surrogate &  0.312 &  3.239 &  3.22 &  3.08 \\
BAMOES &  0.262 & \underline{1.083} &  2.56 & \underline{2.61} \\
Deep BAMOES & \textbf{0.244} & \textbf{0.754} & \underline{2.54} & \textbf{2.23} \\
\bottomrule
\end{tabular}
\end{adjustbox}
\end{table}

\subsection{Main results}
\label{sec:main_results}

Our goal here is to compare different methods with a focus on the quality of uncertainty estimates. 
The main results for Forecasting data are in Table~\ref{tab:rank_table}.
Since Forecasting data has a lot of time series, we decided to count ranks and average it for all pairs of a base model and the corresponding uncertainty estimate for it.
Even higher level aggregation is provided in the teaser table~\ref{tab:teaser_table} with \emph{Our Surrogate} being a pen name for BAMOES and \emph{Best Bootstrap} being a pen name for BSAP BS.
Other insights are provided by CD diagrams in Figures~\ref{fig:cd_miscal}.

All three ways to measure the quality lead to the same conclusion: our BAMOES provides the best results on average.
The Naive surrogate is also a strong baseline, which makes it a sound alternative in some cases.
For specific base models, the picture is the same: BAMOES and Naive Surrogate are significantly better than all other methods for CatBoost and OLS base models.
For the ARIMA base model, BAMOES is also among the cluster of methods significantly better than other methods, while Naive Surrogate underperforms being the worst method according to the ranking.


\begin{figure}[htpb]
   \includegraphics[width=1\columnwidth]{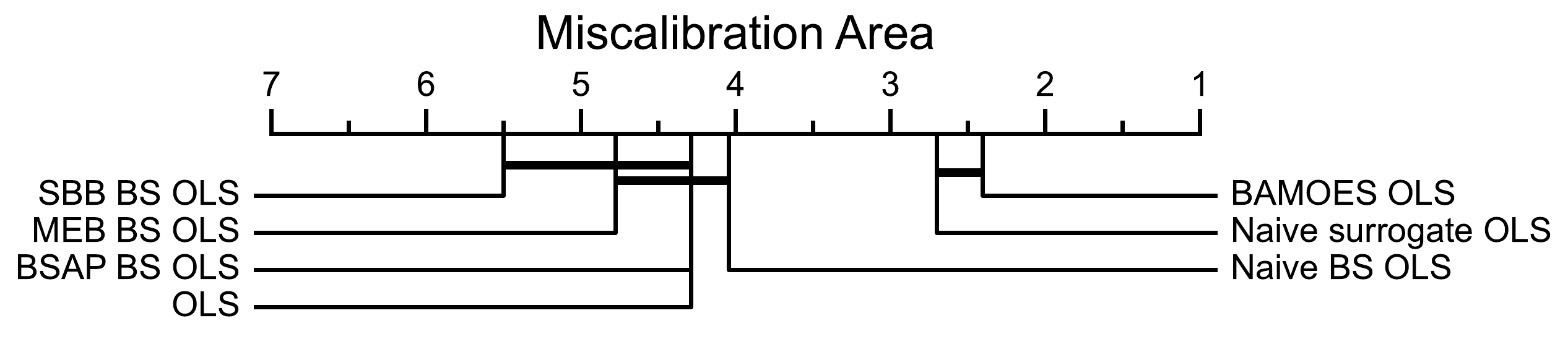}
    \caption{Model comparison of Miscalibration Area on Forecasting data. Here OLS is selected for base model.}
    \label{fig:cd_miscal}
\end{figure}

\subsection{Neural network base model}
We have also conducted experiments with recently proposed time series forecasting model, PatchTST~\cite{Yuqietal-2023-PatchTST} for a large dataset M4 from the original paper.
We have equipped this deep model with uncertainty estimation capability by using our BAMOES surrogate model.
For the sake of comparison we also present results for Naive surrogate, Naive surrogate constructed used embeddings provided by PatchTST (Deep Naive surrogate), BAMOES surrogate and BAMOES surrogate constructed used embeddings provided by PatchTST (Deep BAMOES).
These models have Gaussian process regression at their core core.
Here we also use another strong baseline for uncertainty estimation for deep neural networks based on Monte-Carlo dropout (MC dropout)~\cite{shelmanov2021certain}.

The metrics of UE are presented in Table \ref{tab:m4_dataset}. 
Embeddings-based BAMOES provides consistently better results for both metrics. 

\subsection{Comparison of surrogate models}
\label{sec:surrogate_models_results}

\paragraph{Baseline surrogate model training.} We consider other designs of experiments for training the surrogate model $s$ without changing the log-likelihood loss function but changing the data we use for training.

There are four natural options to replace $\sample$ with an alternative and completely avoid the second term in the loss function~\eqref{eq:gp_loss}:
\begin{enumerate}
    \item (Naive surrogate, Surr I) $\sample = \{(\vecX_i, y_i)\}_{i = 1}^N$,
    \item (Surr II) $\sample_{\hat{f}} = \{(\vecX_i, \hat{f}(\vecX_i))\}_{i = 1}^N$,
    \item (Surr III) $\sample_{\hat{f}} \cup \sample^{\hat{f}} = \{(\vecX_i, \hat{f}(\vecX_i))\}_{i = 1}^N \cup \{(\vecX'_i, \hat{f}(\vecX'_i))\}_{i = 1}^L$
    \item (Surr IV) $\sample \cup \sample^{\hat{f}} = \{(\vecX_i, y_i)\}_{i = 1}^N \cup \{(\vecX'_i, \hat{f}(\vecX'_i))\}_{i = 1}^L$
\end{enumerate}
As one can see, $\sample$ is the original training dataset. 
For (Surr II), we approximate the base model directly, while we lost information about the aleatoric uncertainty presented in the initial dataset. 
For (Surr III) and (Surr IV), we append additional points from input domain $\mathcal{X}$ to the training dataset to better approximate the base model. 
Moreover, for (Surr IV) we use initial targets on the original dataset.
These options are natural baselines with strong empirical evidence behind them.


Results on Forecasting Data are presented in Table \ref{tab:rank_table_surr_all};
critical difference diagram is in Figure \ref{fig:surr_cd_ols_miscal}.
Our BAMOES surrogate outperforms the basic approach which uses no additional points, as well as other approaches with slightly worse results for the CatBoost base model.

\subsection{Selection of hyperparameters}

\paragraph{$C$ and $L$ selection.} Our approach has two key hyperparameters: the weight assigned to the second term in the loss function, denoted as $C$, and the number of supplemental points $L$ in the sample $\sample^{\hat{f}}$. In this subsection, we examine the impact of these variables on model performance. To achieve this, we adjust one hyperparameter while holding the other constant, computing the miscalibration area --- our chosen quality metric --- for each pair.

The outcomes of these experiments for a single dataset are delineated in Figure~\ref{fig:vary_c}. Similar trends hold across all other tested datasets. Moderate values of $C$ and a high number of generated points $L$, approaching the initial sample size $N$, tend to yield superior metrics. However, it's important to note that beyond a certain threshold, further enhancements to performance become negligible.

Based on these observations, we recommend selecting a $C$ value within the range of $[0.5, 1]$ and setting $L \approx N$.




\begin{table}
\scriptsize
\caption{Ranks of uncertainty estimation metrics for surrogate model type of uncertainty estimation aggregated over Forecasting data benchmark}
\label{tab:rank_table_surr_all}
\begin{adjustbox}{width=0.5\textwidth}
\begin{tabular}{lcccc}
\toprule
\begin{tabular}{c}
Uncertainty \\ estimate
\end{tabular}&
\begin{tabular}{c}
Base \\ model
\end{tabular}&
\begin{tabular}{c}
Miscal.\\
Area $\downarrow$
\end{tabular}&
\begin{tabular}{c}
RMSCE $\downarrow$
\end{tabular}&
\begin{tabular}{c}
ENCE $\downarrow$
\end{tabular}\\
\midrule
Built-in & \multirow{6}{*}{OLS} &  4.28 &  4.31 &  4.74 \\
Surr I (Naive) & & \underline{2.73} & \underline{2.66} & \underline{2.75} \\
Surr II  & &  3.94 &  3.93 &  4.0 \\
Surr III  & &  4.31 &  4.37 &  4.12 \\
Surr IV  & &  3.55 &  3.56 &  3.06 \\
BAMOES (ours) & & \textbf{2.19} & \textbf{2.17} & \textbf{2.33} \\
\hline
Built-in & \multirow{6}{*}{ARIMA} &  3.669 &  3.653 &  3.898 \\
Surr I (Naive) & &  3.975 &  3.924 &  4.636 \\
Surr II  & & \underline{3.314} & \underline{3.229} &  3.246 \\
Surr III  & &  3.551 &  3.559 &  3.271 \\
Surr IV  & &  3.602 &  3.678 & \underline{3.068} \\
BAMOES (ours) & & \textbf{2.89} & \textbf{2.958} & \textbf{2.881} \\
\hline
Built-in & \multirow{6}{*}{CatBoost} &  4.681 &  4.655 &  3.966 \\
Surr I (Naive) & & \textbf{2.655} & \textbf{2.629} & \underline{2.586} \\
Surr II  & &  3.5 &  3.509 &  3.845 \\
Surr III  & &  3.94 &  3.94 &  4.526 \\
Surr IV  & &  3.362 &  3.388 &  3.526 \\
BAMOES (ours) & & \underline{2.862} & \underline{2.879} & \textbf{2.552} \\
\bottomrule
\end{tabular}
\end{adjustbox}
\end{table}

\begin{figure}[h]
    \centering
    \includegraphics[width=0.45\textwidth]{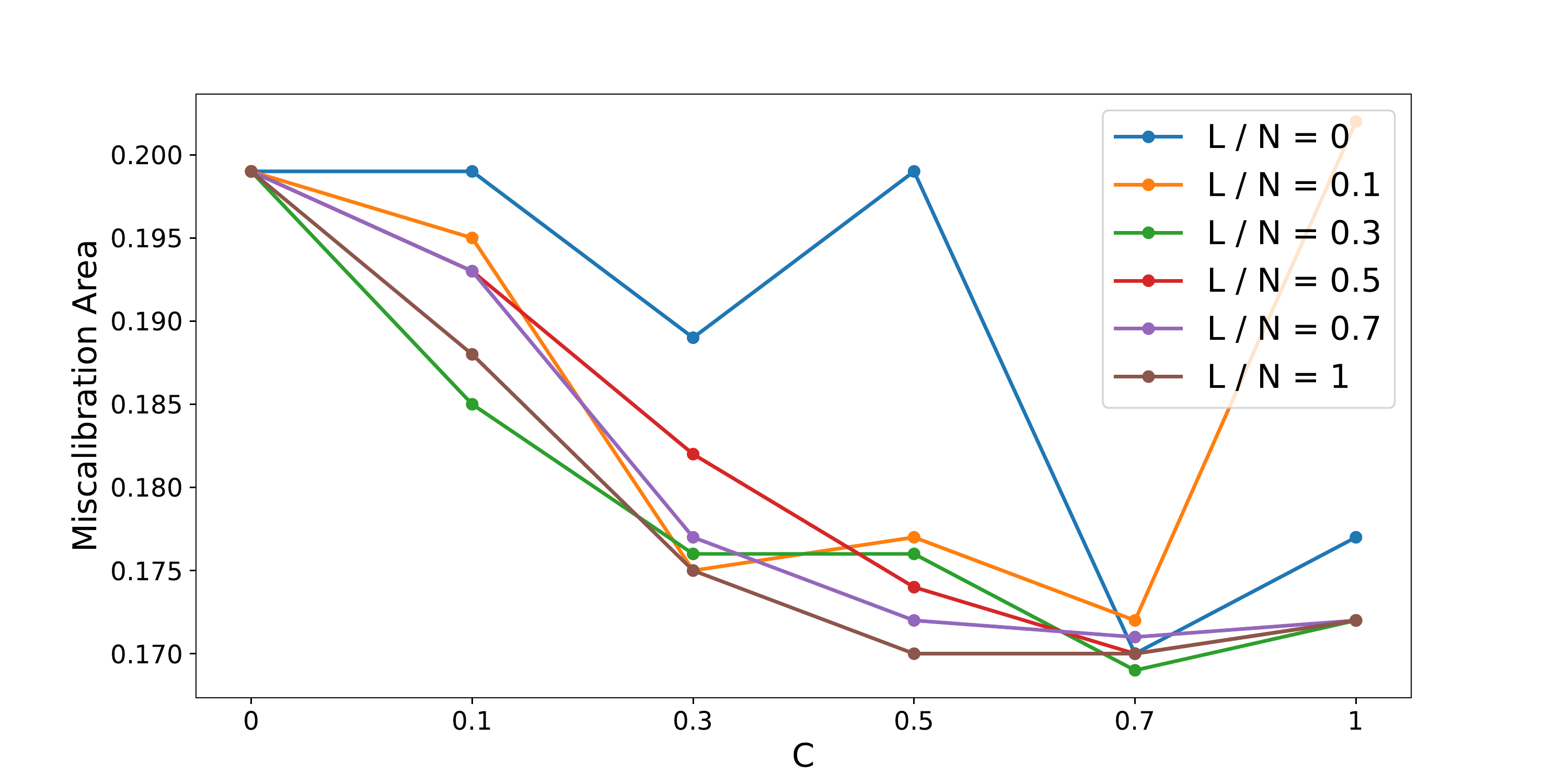}
    \caption{Dependence of the surrogate uncertainty estimate quality metric Miscalibration area on the hyperparameter $C$ for different numbers of generated points}
    \label{fig:vary_c}
\end{figure}


%% file: tex/appendix.tex
\subsection{BAMOES pseudocode}
\label{sec:pseudocode}

In Algorithm \ref{alg:surr_ind} we provide a pseudocode for our BAMOES approach.
We assume, that a surrogate $s(\vecX)$ has a corresponding kernel $k_s(\vecX, \vecX')$:
\begin{algorithm}
\caption{BAMOES training}\label{alg:surr_ind}
\begin{algorithmic}[1]
\Require Data ($X$, $\vecY$), base model $\hat{f}$, sample size $L$, loss weight $C$, number of epochs $M$
\State Sample uniformly $L$ points of interest:
\State  $X'=\{\vecX_i'\}_{i=1}^L$, $x_{ij}' \sim U(\mathrm{low} = \min(X_{\cdot j}), \mathrm{high} = \max(X_{\cdot j}))$
\State Get base model predictions $\vecY' = \{\hat{f}(\vecX_i')\}_{i=1}^L$
\State Define GPR surrogate model $s(\vecX)$
\State{\textbf{for all} epoch $\in \{1,\dots, M\}$}
\State \phantom{aaaa} Compute the log likelihood $l(s) = \log p(\vecY| X, s)$
\State \phantom{aaaa} Compute surrogate predictions $\{s(\vecX_i')\}_{i=1}^L$
\State \phantom{aaaa} Compute $e(s) = \sum_{i = 1}^{L} (s(\vecX_i') - \hat{f}(\vecX_i'))^2$
\State \phantom{aaaa} Calculate loss: $\text{loss}(s) = -(1 - C) l(s) + C e(s)$
\State \phantom{aaaa} Update $k_s(\cdot\,,\cdot)$ to minimize loss via Adam for the gradients of $\text{loss}(s)$
\State \Return Surrogate $s(\vecX)$
\end{algorithmic}
\end{algorithm}

    






\subsection{Additional experiments}
\label{sec:add_experiments}

\subsubsection{Quality of uncertainty estimation for datasets A and B}
We also provide metrics for separate datasets A and B to give a taste of what is not only the rank but also related differences in metrics.
Tables~\ref{tab:ds_A_table} and~\ref{tab:ds_B_table} provide detailed uncertainty estimation quality metrics for Dataset A and Dataset B, respectively. 
Again, our BAMOES provides superior uncertainty estimation metrics in many cases.

\subsubsection{Quality of uncertainty estimation for different base models}

The results of the base models are given in Table~\ref{tab:base_models_results_loans_short} for Dataset A and in Table~\ref{tab:base_models_results_mortgages_short} for Dataset B. We present both regression quality and quality of uncertainty estimation.
For all considered forecasting horizons Gaussian process regression (GPR) approach is in the top two for both considered uncertainty estimation quality metrics. Consequently, it is a good candidate for a surrogate-based uncertainty estimate, if we can make GPR close to a base model.

\begin{table}[H]
\scriptsize
\caption{Regression and calibration metrics for base models on Dataset A. Best results are highlighted with \textbf{bold} font, second best results are \underline{underscored}
}
\label{tab:base_models_results_loans_short}
\begin{adjustbox}{width=0.5\textwidth}
\begin{tabular}{cccccc}
\toprule
\begin{tabular}{c}
Model
\end{tabular}&
\begin{tabular}{c}
Horizon
\end{tabular}&
\begin{tabular}{c}
RMSE
\end{tabular}&
\begin{tabular}{c}
RMSCE
\end{tabular}&
\begin{tabular}{c}
Miscal.\\
Area
\end{tabular}&
\begin{tabular}{c}
ENCE
\end{tabular}\\
\midrule
GPR & \multirow{6}{*}{3} &  14712.701 & \textbf{0.319} & \textbf{0.262} & \textbf{0.322} \\
OLS & & \underline{11933.3} &  0.57 &  0.495 &  3.285 \\
CatBoost & &  14173.191 &  0.57 &  0.495 &  0.999 \\
ARIMA & & \textbf{9035.37} & \underline{0.333} & \underline{0.286} & \underline{0.695} \\
SQR PL & & 8152.979 & 0.131 & 0.106 & 0.489 \\
SQR CL & & 6572.678 & 0.236 & 0.195 & 0.912 \\
\hline
GPR & \multirow{6}{*}{6} &  12725.005 & \textbf{0.199} & \textbf{0.164} & \textbf{0.302} \\
OLS & & \underline{10917.502} &  0.558 &  0.487 &  2.604 \\
CatBoost & & \textbf{6206.29} &  0.57 &  0.495 &  1.0 \\
ARIMA & &  17727.489 & \underline{0.278} & \underline{0.238} & \underline{0.593} \\
SQR PL & & 11580.078 & 0.191 & 0.144 & 0.404 \\
SQR CL & & 9959.138 & 0.173 & 0.125 & 0.735 \\
\hline
GPR & \multirow{6}{*}{12} &  19160.008 & \textbf{0.377} & \textbf{0.322} & \textbf{0.645} \\
OLS & & \textbf{13849.332} &  0.518 &  0.46 &  2.696 \\
CatBoost & &  18604.156 &  0.57 &  0.495 &  0.999 \\
ARIMA & & \underline{14781.831} & \underline{0.405} & \underline{0.351} & \underline{0.801} \\
SQR PL & & 7522.467 & 0.12 & 0.103 & 0.418 \\
SQR CL & & 11053.891 & 0.213 & 0.183 & 0.291 \\
\hline
GPR & \multirow{6}{*}{18} & \textbf{20955.192} & \textbf{0.074} & \textbf{0.061} & \underline{0.707} \\
OLS & &  21764.586 & \underline{0.264} & \underline{0.226} &  1.662 \\
CatBoost & & \underline{21170.805} &  0.57 &  0.495 &  0.999 \\
ARIMA & &  29662.23 &  0.36 &  0.314 & \textbf{0.598} \\
SQR PL & & 27518.828 & 0.363 & 0.326 & 3.466 \\
SQR CL & & 22314.858 & 0.089 & 0.073 & 0.92 \\
\hline
GPR & \multirow{6}{*}{24} & \underline{21493.571} & \underline{0.266} & \underline{0.239} & \textbf{0.423} \\
OLS & & \textbf{16362.093} & \textbf{0.184} & \textbf{0.167} &  0.824 \\
CatBoost & &  33352.452 &  0.563 &  0.489 &  0.987 \\
ARIMA & &  25524.539 &  0.313 &  0.269 & \underline{0.689} \\
SQR PL & & 26595. & 0.363 & 0.326 & 3.466 \\
SQR CL & & 22314.858 & 0.089 & 0.073 & 0.92 \\
\bottomrule
\end{tabular}
\end{adjustbox}
\end{table}

\begin{table}[H]
\scriptsize
\caption{Regression and calibration metrics for base models on Dataset B. Best results are highlighted with \textbf{bold} font, second best results are \underline{underscored}
}
\label{tab:base_models_results_mortgages_short}
\begin{adjustbox}{width=0.5\textwidth}
\begin{tabular}{cccccc}
\toprule
\begin{tabular}{c}
Model
\end{tabular}&
\begin{tabular}{c}
Horizon
\end{tabular}&
\begin{tabular}{c}
RMSE
\end{tabular}&
\begin{tabular}{c}
RMSCE
\end{tabular}&
\begin{tabular}{c}
Miscal.\\
Area
\end{tabular}&
\begin{tabular}{c}
ENCE
\end{tabular}\\
\midrule
GPR & \multirow{4}{*}{3} & \textbf{0.015} & \textbf{0.203} & \textbf{0.171} & \textbf{0.536} \\
OLS & &  0.016 & \underline{0.254} & \underline{0.216} &  1.839 \\
CatBoost & &  0.019 &  0.57 &  0.495 &  226.223 \\
ARIMA & & \textbf{0.015} &  0.352 &  0.295 & \underline{0.733} \\
\hline
GPR & \multirow{4}{*}{6} & \textbf{0.012} & \textbf{0.258} & \textbf{0.223} & \textbf{0.554} \\
OLS & & \underline{0.013} & \underline{0.319} & \underline{0.268} &  1.279 \\
CatBoost & &  0.02 &  0.57 &  0.495 &  239.636 \\
ARIMA & &  0.017 &  0.391 &  0.335 & \underline{0.785} \\
\hline
GPR & \multirow{4}{*}{12} &  0.018 & \textbf{0.135} & \textbf{0.117} & \textbf{0.319} \\
OLS & & \underline{0.017} & \underline{0.259} & \underline{0.241} &  1.561 \\
CatBoost & & \textbf{0.016} &  0.57 &  0.495 &  1012.208 \\
ARIMA & &  0.022 &  0.43 &  0.378 & \underline{0.831} \\
\hline
GPR & \multirow{4}{*}{18} & \textbf{0.028} & \textbf{0.072} & \textbf{0.058} & \textbf{0.376} \\
OLS & & \textbf{0.028} &  0.375 &  0.316 &  2.425 \\
CatBoost & &  0.032 &  0.57 &  0.495 &  395.901 \\
ARIMA & &  0.049 & \underline{0.239} & \underline{0.201} & \underline{0.555} \\
\hline
GPR & \multirow{4}{*}{24} & \underline{0.029} & \textbf{0.074} & \textbf{0.066} & \textbf{0.431} \\
OLS & & \textbf{0.028} & \underline{0.384} & \underline{0.35} &  1.801 \\
CatBoost & & \underline{0.029} &  0.535 &  0.455 &  2250.228 \\
ARIMA & &  0.063 &  0.435 &  0.374 & \underline{0.836} \\
\bottomrule
\end{tabular}
\end{adjustbox}
\end{table}

\begin{table}
\scriptsize
\caption{Quality metrics for regression and uncertainty estimation on Dataset A}
\label{tab:ds_A_table}
\begin{adjustbox}{width=0.5\textwidth}
\begin{tabular}{lccccc}
\toprule
\begin{tabular}{c}
Uncertainty \\ estimate
\end{tabular}&
\begin{tabular}{c}
Base model
\end{tabular}&
\begin{tabular}{c}
RMSE
\end{tabular}&
\begin{tabular}{c}
Miscal.\\
Area
\end{tabular}&
\begin{tabular}{c}
RMSCE
\end{tabular}&
\begin{tabular}{c}
ENCE
\end{tabular}\\
\midrule
Built-in & \multirow{7}{*}{OLS} & \textbf{16362.093} &  0.167 &  0.184 &  0.824 \\
Naive BS & &  16696.277 &  0.204 &  0.229 &  1.189 \\
MEB BS & &  16512.081 &  0.32 &  0.358 &  2.213 \\
SBB BS & &  17394.776 & \underline{0.054} & \underline{0.067} &  0.476 \\
BSAP BS & &  20228.613 &  0.242 &  0.276 &  2.596 \\
Naive Surrogate & & \textbf{16362.093} & \textbf{0.028} & \textbf{0.036} & \textbf{0.234} \\
BAMOES (ours) & & \textbf{16362.093} &  0.084 &  0.098 & \underline{0.384} \\
\hline
Built-in & \multirow{7}{*}{ARIMA} &  25524.539 &  0.269 &  0.313 &  0.689 \\
Naive BS & & \textbf{16296.585} & \underline{0.144} & \underline{0.16} & \textbf{0.61} \\
MEB BS & &  26411.274 &  0.45 &  0.513 &  4.075 \\
SBB BS & &  32248.943 &  0.43 &  0.479 &  2.078 \\
BSAP BS & & \underline{23403.858} & \textbf{0.099} & \textbf{0.115} & \underline{0.612} \\
Naive surrogate & &  25524.539 &  0.292 &  0.33 &  0.823 \\
BAMOES (ours) & &  25524.539 &  0.358 &  0.4 &  1.317 \\
\hline
Built-in & \multirow{7}{*}{CatBoost} &  33352.452 &  0.489 &  0.563 & \textbf{0.987} \\
Naive BS & & \underline{31155.193} &  0.456 &  0.535 &  14.201 \\
MEB BS & &  31949.12 &  0.466 &  0.538 &  9.144 \\
SBB BS & &  33602.92 &  0.495 &  0.57 &  22.932 \\
BSAP BS & & \textbf{30253.746} &  0.469 &  0.54 &  13.866 \\
Naive surrogate & &  33352.452 & \textbf{0.354} & \textbf{0.394} & \underline{1.374} \\
BAMOES (ours) & &  33352.452 & \underline{0.373} & \underline{0.413} &  1.437 \\
\bottomrule
\end{tabular}
\end{adjustbox}
\end{table}

\begin{table}
\scriptsize
\caption{Quality metrics for regression and uncertainty estimation on Dataset B}
\label{tab:ds_B_table}
\begin{adjustbox}{width=0.5\textwidth}
\begin{tabular}{lccccc}
\toprule
\begin{tabular}{c}
Uncertainty \\ estimate
\end{tabular}&
\begin{tabular}{c}
Base model
\end{tabular}&
\begin{tabular}{c}
RMSE
\end{tabular}&
\begin{tabular}{c}
Miscal.\\
Area
\end{tabular}&
\begin{tabular}{c}
RMSCE
\end{tabular}&
\begin{tabular}{c}
ENCE
\end{tabular}\\
\midrule
Built-in & \multirow{7}{*}{OLS} & \textbf{0.028} &  0.35 &  0.384 &  1.801 \\
Naive BS & &  0.029 &  0.363 &  0.406 &  2.922 \\
MEB BS & & \textbf{0.028} &  0.408 &  0.452 &  2.507 \\
SBB BS & & \textbf{0.028} &  0.363 &  0.399 &  1.913 \\
BSAP BS & &  0.03 &  0.314 &  0.352 &  2.12 \\
Naive surrogate & & \textbf{0.028} & \underline{0.058} & \underline{0.067} & \underline{0.416} \\
BAMOES (ours) & & \textbf{0.028} & \textbf{0.054} & \textbf{0.063} & \textbf{0.348} \\
\hline
Built-in & \multirow{7}{*}{ARIMA} &  0.063 &  0.374 &  0.435 &  0.836 \\
Naive BS & & \textbf{0.029} &  0.41 &  0.451 &  2.647 \\
MEB BS & &  0.039 &  0.272 &  0.295 &  1.057 \\
SBB BS & &  0.065 & \textbf{0.085} & \textbf{0.094} & \underline{0.465} \\
BSAP BS & & \underline{0.035} &  0.206 &  0.245 &  1.177 \\
Naive surrogate & &  0.063 &  0.37 &  0.407 &  1.561 \\
BAMOES (ours) & &  0.063 & \underline{0.086} & \underline{0.101} & \textbf{0.342} \\
\hline
Built-in & \multirow{7}{*}{CatBoost} &  0.029 &  0.455 &  0.535 &  2250.228 \\
Naive BS & & \underline{0.028} &  0.225 &  0.263 &  3.336 \\
MEB BS & &  0.03 &  0.29 &  0.345 &  2.417 \\
SBB BS & &  0.031 &  0.334 &  0.386 &  4.092 \\
BSAP BS & & \textbf{0.026} &  0.178 &  0.215 &  1.934 \\
Naive surrogate & &  0.029 & \underline{0.046} & \textbf{0.055} & \textbf{0.479} \\
BAMOES (ours) & &  0.029 & \textbf{0.045} & \textbf{0.055} & \underline{0.486} \\
\bottomrule
\end{tabular}
\end{adjustbox}
\end{table}

\subsubsection{Critical difference diagrams}
\label{sec:crit_diag}

Let us examine additional critical difference diagrams.

In the case of the OLS base model (Figures~\ref{fig:cd_miscal} and~\ref{fig:cd_rmse}), we can see that MEB BS OLS and Naive BS OLS are not significantly different in terms of Miscalibration Area. Moreover, other methods are in a clique. In terms of RMSE, there are three cliques: first clique (Naive BS OLS), second clique (MEB BS OLS, OLS, Surr I OLS, Surr I OLS), and third clique (SBB BS OLS, BSAP BS OLS). In Table~\ref{tab:rank_table} we can see that some methods have metric values equal to or close enough to the base model (OLS).

In the case of the ARIMA base model (Figures~\ref{fig:cd_miscal} and~\ref{fig:cd_rmse}), we see that there are three cliques in terms of Miscalibarion Area and three other cliques in terms of RMSE.

In the case of the CatBoost base model (Figure~\ref{fig:cd_miscal}), we see that there are three cliques in the term of Miscalibration Area.

\begin{figure}[H]
    \centering
\includegraphics[width=1\columnwidth]{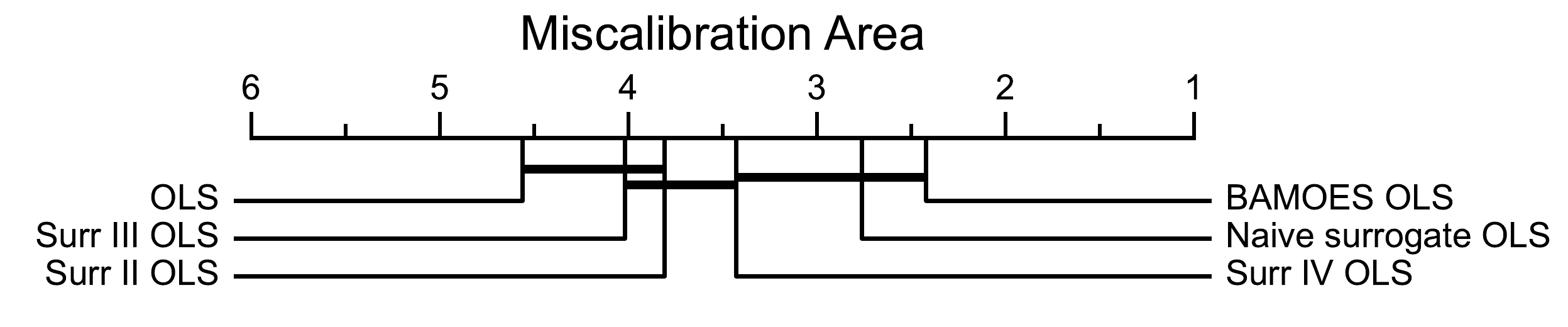}
    \caption{OLS base surrogate model comparison of Miscal.
Area on Forecasting data}
    \label{fig:surr_cd_ols_miscal}
\end{figure}

\begin{figure}[H]
    \begin{subcaptionblock}{1.\columnwidth}
        \includegraphics[width=1\columnwidth]{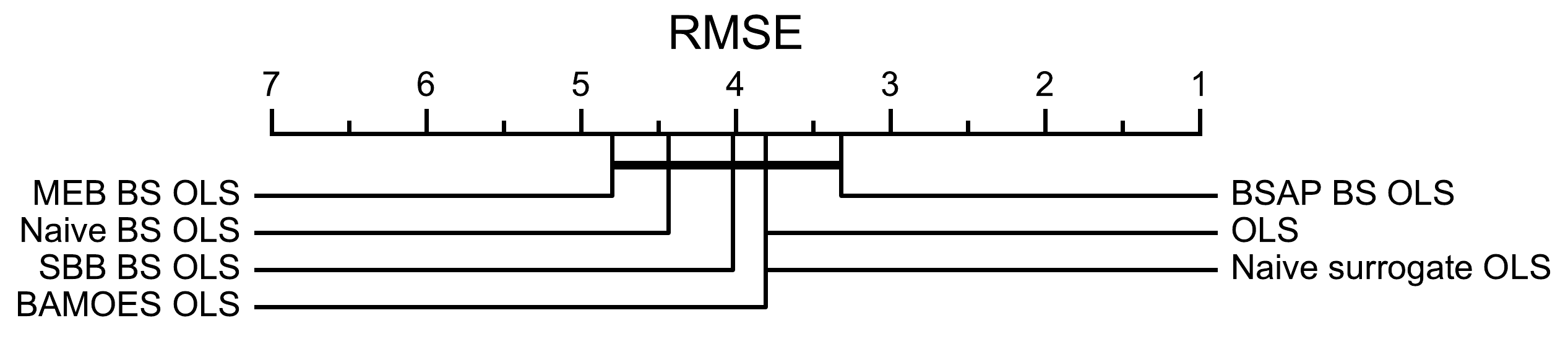}
    \end{subcaptionblock} \\ %
    \begin{subcaptionblock}{1.\columnwidth}
        \includegraphics[width=1\columnwidth]{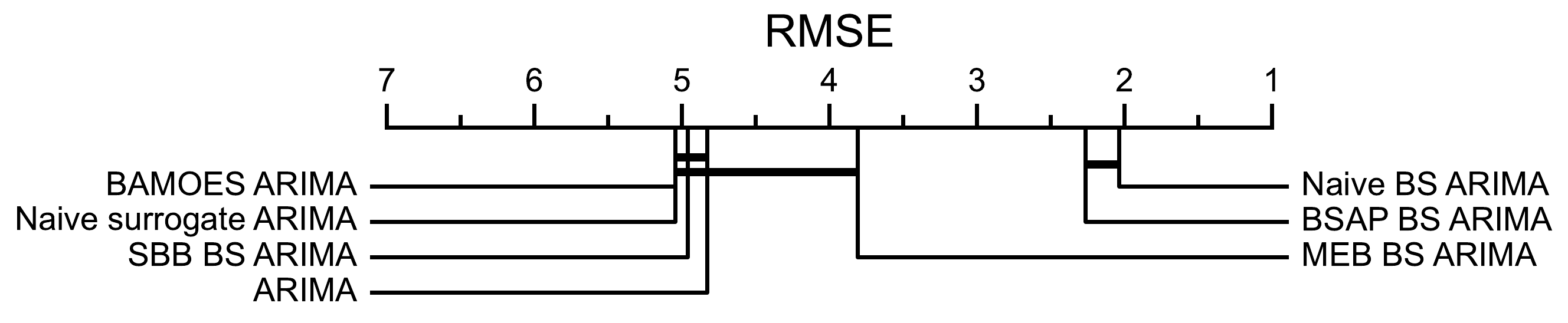}
    \end{subcaptionblock} \\ %
    \caption{Model comparison of Miscalibration Area and RMSE on Forecasting data}
    \label{fig:cd_rmse}
\end{figure}

%% file: final_conference.bbl
\begin{thebibliography}{10}
\providecommand{\url}[1]{#1}
\csname url@samestyle\endcsname
\providecommand{\newblock}{\relax}
\providecommand{\bibinfo}[2]{#2}
\providecommand{\BIBentrySTDinterwordspacing}{\spaceskip=0pt\relax}
\providecommand{\BIBentryALTinterwordstretchfactor}{4}
\providecommand{\BIBentryALTinterwordspacing}{\spaceskip=\fontdimen2\font plus
\BIBentryALTinterwordstretchfactor\fontdimen3\font minus \fontdimen4\font\relax}
\providecommand{\BIBforeignlanguage}[2]{{%
\expandafter\ifx\csname l@#1\endcsname\relax
\typeout{** WARNING: IEEEtran.bst: No hyphenation pattern has been}%
\typeout{** loaded for the language `#1'. Using the pattern for}%
\typeout{** the default language instead.}%
\else
\language=\csname l@#1\endcsname
\fi
#2}}
\providecommand{\BIBdecl}{\relax}
\BIBdecl

\bibitem{Kompa2021}
\BIBentryALTinterwordspacing
B.~Kompa, J.~Snoek, and A.~L. Beam, ``Second opinion needed: communicating uncertainty in medical machine learning,'' \emph{npj Digital Medicine}, vol.~4, no.~1, p.~4, Jan 2021. [Online]. Available: \url{https://doi.org/10.1038/s41746-020-00367-3}
\BIBentrySTDinterwordspacing

\bibitem{9525313}
D.~Feng, A.~Harakeh, S.~L. Waslander, and K.~Dietmayer, ``A review and comparative study on probabilistic object detection in autonomous driving,'' \emph{IEEE Transactions on Intelligent Transportation Systems}, vol.~23, no.~8, pp. 9961--9980, 2022.

\bibitem{lahiri2003resampling}
S.~Lahiri and S.~Lahiri, \emph{Resampling methods for dependent data}.\hskip 1em plus 0.5em minus 0.4em\relax New York: Springer Science \& Business Media, 2003.

\bibitem{lim2021time}
B.~Lim and S.~Zohren, ``Time-series forecasting with deep learning: a survey,'' \emph{Philosophical Transactions of the Royal Society A}, vol. 379, no. 2194, p. 20200209, 2021.

\bibitem{zhu2017deep}
L.~Zhu and N.~Laptev, ``Deep and confident prediction for time series at {U}ber,'' in \emph{2017 IEEE International Conference on Data Mining Workshops (ICDMW)}.\hskip 1em plus 0.5em minus 0.4em\relax IEEE, 2017, pp. 103--110.

\bibitem{2210.17030}
Y.~Fujimoto, K.~Nakagawa, K.~Imajo, and K.~Minami, ``Uncertainty aware trader-company method: Interpretable stock price prediction capturing uncertainty,'' 2022.

\bibitem{lakshminarayanan2017simple}
B.~Lakshminarayanan, A.~Pritzel, and C.~Blundell, ``Simple and scalable predictive uncertainty estimation using deep ensembles,'' \emph{NeurIPS}, vol.~30, 2017.

\bibitem{shelmanov2021certain}
A.~Shelmanov, E.~Tsymbalov, D.~Puzyrev, K.~Fedyanin, A.~Panchenko, and M.~Panov, ``How certain is your transformer?'' in \emph{Proceedings of the 16th Conference of the European Chapter of the Association for Computational Linguistics: Main Volume}, 2021, pp. 1833--1840.

\bibitem{koziel2013surrogate}
S.~Koziel and L.~Leifsson, \emph{Surrogate-based modeling and optimization}.\hskip 1em plus 0.5em minus 0.4em\relax Springer, 2013.

\bibitem{williams2006gaussian}
C.~K. Williams and C.~E. Rasmussen, \emph{Gaussian processes for machine learning}.\hskip 1em plus 0.5em minus 0.4em\relax Cambridge, MA: MIT press, 2006, vol.~2.

\bibitem{wilson2016deep}
A.~G. Wilson, Z.~Hu, R.~Salakhutdinov, and E.~P. Xing, ``Deep kernel learning,'' in \emph{AISTATS}.\hskip 1em plus 0.5em minus 0.4em\relax PMLR, 2016, pp. 370--378.

\bibitem{NEURIPS2020_543e8374}
\BIBentryALTinterwordspacing
J.~Liu, Z.~Lin, S.~Padhy, D.~Tran, T.~Bedrax~Weiss, and B.~Lakshminarayanan, ``Simple and principled uncertainty estimation with deterministic deep learning via distance awareness,'' in \emph{NeurIPS}, H.~Larochelle, M.~Ranzato, R.~Hadsell, M.~Balcan, and H.~Lin, Eds., vol.~33.\hskip 1em plus 0.5em minus 0.4em\relax Virtual: Curran Associates, Inc., 2020, pp. 7498--7512. [Online]. Available: \url{https://proceedings.neurips.cc/paper/2020/file/543e83748234f7cbab21aa0ade66565f-Paper.pdf}
\BIBentrySTDinterwordspacing

\bibitem{shao1996bootstrap}
J.~Shao, ``Bootstrap model selection,'' \emph{Journal of the American statistical Association}, vol.~91, no. 434, pp. 655--665, 1996.

\bibitem{fasiolo2021fast}
M.~Fasiolo, S.~N. Wood, M.~Zaffran, R.~Nedellec, and Y.~Goude, ``Fast calibrated additive quantile regression,'' \emph{Journal of the American Statistical Association}, vol. 116, no. 535, pp. 1402--1412, 2021.

\bibitem{salem2020prediction}
T.~S. Salem, H.~Langseth, and H.~Ramampiaro, ``Prediction intervals: Split normal mixture from quality-driven deep ensembles,'' in \emph{UAI}.\hskip 1em plus 0.5em minus 0.4em\relax PMLR, 2020, pp. 1179--1187.

\bibitem{box2015time}
G.~E. Box, G.~M. Jenkins, G.~C. Reinsel, and G.~M. Ljung, \emph{Time series analysis: forecasting and control}.\hskip 1em plus 0.5em minus 0.4em\relax John Wiley \& Sons, 2015.

\bibitem{dorogush2018catboost}
A.~V. Dorogush, V.~Ershov, and A.~Gulin, ``Catboost: gradient boosting with categorical features support,'' \emph{arXiv preprint arXiv:1810.11363}, 2018.

\bibitem{snyder2001prediction}
R.~D. Snyder, J.~K. Ord, and A.~B. Koehler, ``Prediction intervals for arima models,'' \emph{Journal of Business \& Economic Statistics}, vol.~19, no.~2, pp. 217--225, 2001.

\bibitem{malinin2020uncertainty}
A.~Malinin, L.~Prokhorenkova, and A.~Ustimenko, ``Uncertainty in gradient boosting via ensembles,'' \emph{arXiv preprint arXiv:2006.10562}, 2020.

\bibitem{duan2020ngboost}
T.~Duan, A.~Anand, D.~Y. Ding, K.~K. Thai, S.~Basu, A.~Ng, and A.~Schuler, ``Ngboost: Natural gradient boosting for probabilistic prediction,'' in \emph{ICML}.\hskip 1em plus 0.5em minus 0.4em\relax PMLR, 2020, pp. 2690--2700.

\bibitem{roy2011comprehensive}
C.~J. Roy and W.~L. Oberkampf, ``A comprehensive framework for verification, validation, and uncertainty quantification in scientific computing,'' \emph{Computer methods in applied mechanics and engineering}, vol. 200, no. 25-28, pp. 2131--2144, 2011.

\bibitem{abdar2021review}
M.~Abdar, F.~Pourpanah, S.~Hussain \emph{et~al.}, ``A review of uncertainty quantification in deep learning: Techniques, applications and challenges,'' \emph{Information Fusion}, vol.~76, pp. 243--297, 2021.

\bibitem{liu2019accurate}
J.~Liu, J.~Paisley, M.-A. Kioumourtzoglou, and B.~Coull, ``Accurate uncertainty estimation and decomposition in ensemble learning,'' \emph{NeurIPS}, vol.~32, 2019.

\bibitem{kunsch1989jacknife}
\BIBentryALTinterwordspacing
H.~R. Kunsch, ``{The Jackknife and the Bootstrap for General Stationary Observations},'' \emph{The Annals of Statistics}, vol.~17, no.~3, pp. 1217 -- 1241, 1989. [Online]. Available: \url{https://doi.org/10.1214/aos/1176347265}
\BIBentrySTDinterwordspacing

\bibitem{politis1994stationary}
\BIBentryALTinterwordspacing
D.~N. Politis and J.~P. Romano, ``The stationary bootstrap,'' \emph{Journal of the American Statistical Association}, vol.~89, no. 428, pp. 1303--1313, 1994. [Online]. Available: \url{http://www.jstor.org/stable/2290993}
\BIBentrySTDinterwordspacing

\bibitem{paparoditis2001tapered}
\BIBentryALTinterwordspacing
E.~Paparoditis and D.~N. Politis, ``Tapered block bootstrap,'' \emph{Biometrika}, vol.~88, no.~4, pp. 1105--1119, 2001. [Online]. Available: \url{http://www.jstor.org/stable/2673704}
\BIBentrySTDinterwordspacing

\bibitem{tagasovska2019single}
N.~Tagasovska and D.~Lopez-Paz, ``Single-model uncertainties for deep learning,'' \emph{NeurIPS}, vol.~32, 2019.

\bibitem{chung2021beyond}
Y.~Chung, W.~Neiswanger, I.~Char, and J.~Schneider, ``Beyond pinball loss: Quantile methods for calibrated uncertainty quantification,'' \emph{NeurIPS}, vol.~34, pp. 10\,971--10\,984, 2021.

\bibitem{zaytsev2018interpolation}
A.~Zaytsev, E.~Romanenkova, and D.~Ermilov, ``Interpolation error of gaussian process regression for misspecified case,'' in \emph{Conformal and Probabilistic Prediction and Applications}.\hskip 1em plus 0.5em minus 0.4em\relax PMLR, 2018, pp. 83--95.

\bibitem{neal2012bayesian}
R.~M. Neal, \emph{Bayesian learning for neural networks}.\hskip 1em plus 0.5em minus 0.4em\relax New York: Springer Science \& Business Media, 2012, vol. 118.

\bibitem{damianou2013deep}
A.~Damianou and N.~D. Lawrence, ``Deep {G}aussian processes,'' in \emph{AISTATS}, PMLR.\hskip 1em plus 0.5em minus 0.4em\relax Scottsdale, AZ, USA: JMLR Workshop and Conference Proceedings, 2013, pp. 207--215.

\bibitem{van2011information}
A.~Van Der~Vaart and H.~Van~Zanten, ``Information rates of nonparametric gaussian process methods.'' \emph{Journal of Machine Learning Research}, vol.~12, no.~6, 2011.

\bibitem{zaytsev2017minimax}
A.~Zaytsev and E.~Burnaev, ``Minimax approach to variable fidelity data interpolation,'' in \emph{AISTATS}.\hskip 1em plus 0.5em minus 0.4em\relax PMLR, 2017, pp. 652--661.

\bibitem{roberts2013gaussian}
S.~Roberts, M.~Osborne, M.~Ebden, S.~Reece, N.~Gibson, and S.~Aigrain, ``Gaussian processes for time-series modelling,'' \emph{Philosophical Transactions of the Royal Society A: Mathematical, Physical and Engineering Sciences}, vol. 371, no. 1984, p. 20110550, 2013.

\bibitem{gutjahr2012sparse}
T.~Gutjahr, H.~Ulmer, and C.~Ament, ``Sparse gaussian processes with uncertain inputs for multi-step ahead prediction,'' \emph{IFAC Proceedings Volumes}, vol.~45, no.~16, pp. 107--112, 2012.

\bibitem{miyato2018spectral}
T.~Miyato, T.~Kataoka, M.~Koyama, and Y.~Yoshida, ``Spectral normalization for generative adversarial networks,'' in \emph{International Conference on Learning Representations}.\hskip 1em plus 0.5em minus 0.4em\relax Vancouver Convention Center, Vancouver Canada: ICLR, 2018, pp. 1--26.

\bibitem{mi2022training}
L.~Mi, H.~Wang, Y.~Tian, H.~He, and N.~N. Shavit, ``Training-free uncertainty estimation for dense regression: Sensitivity as a surrogate,'' in \emph{Proceedings of the AAAI Conference on Artificial Intelligence}, vol.~36.\hskip 1em plus 0.5em minus 0.4em\relax Palo Alto, California USA: AAAI Press, 2022, pp. 10\,042--10\,050.

\bibitem{tsymbalov2019deeper}
E.~Tsymbalov, S.~Makarychev, A.~Shapeev, and M.~Panov, ``Deeper connections between neural networks and gaussian processes speed-up active learning,'' in \emph{Proceedings of the 28th International Joint Conference on Artificial Intelligence}, ser. IJCAI'19.\hskip 1em plus 0.5em minus 0.4em\relax Macao, China: AAAI Press, 2019, p. 3599–3605.

\bibitem{WANG2014167}
\BIBentryALTinterwordspacing
C.~Wang, Q.~Duan, W.~Gong, A.~Ye, Z.~Di, and C.~Miao, ``An evaluation of adaptive surrogate modeling based optimization with two benchmark problems,'' \emph{Environmental Modelling \& Software}, vol.~60, pp. 167--179, 2014. [Online]. Available: \url{https://www.sciencedirect.com/science/article/pii/S1364815214001698}
\BIBentrySTDinterwordspacing

\bibitem{quinonero2005unifying}
J.~Quinonero-Candela and C.~E. Rasmussen, ``A unifying view of sparse approximate gaussian process regression,'' \emph{The Journal of Machine Learning Research}, vol.~6, pp. 1939--1959, 2005.

\bibitem{godahewa2021monash}
R.~Godahewa, C.~Bergmeir, G.~I. Webb, R.~J. Hyndman, and P.~Montero-Manso, ``Monash time series forecasting archive,'' \emph{arXiv preprint arXiv:2105.06643}, 2021.

\bibitem{chung2021uncertainty}
Y.~Chung, I.~Char, H.~Guo, J.~Schneider, and W.~Neiswanger, ``Uncertainty toolbox: an open-source library for assessing, visualizing, and improving uncertainty quantification,'' \emph{arXiv preprint arXiv:2109.10254}, vol. 2109, pp. 1--8, 2021.

\bibitem{tran2020methods}
K.~Tran, W.~Neiswanger, J.~Yoon, Q.~Zhang, E.~Xing, and Z.~W. Ulissi, ``Methods for comparing uncertainty quantifications for material property predictions,'' \emph{Machine Learning: Science and Technology}, vol.~1, no.~2, p. 025006, 2020.

\bibitem{ence}
\BIBentryALTinterwordspacing
D.~Levi, L.~Gispan, N.~Giladi, and E.~Fetaya, ``Evaluating and calibrating uncertainty prediction in regression tasks,'' \emph{Sensors}, vol.~22, no.~15, 2022. [Online]. Available: \url{https://www.mdpi.com/1424-8220/22/15/5540}
\BIBentrySTDinterwordspacing

\bibitem{demvsar2006statistical}
J.~Dem{\v{s}}ar, ``Statistical comparisons of classifiers over multiple data sets,'' \emph{The Journal of Machine learning research}, vol.~7, pp. 1--30, 2006.

\bibitem{IsmailFawaz2018deep}
H.~Ismail~Fawaz, G.~Forestier, J.~Weber, L.~Idoumghar, and P.-A. Muller, ``Deep learning for time series classification: a review,'' \emph{Data Mining and Knowledge Discovery}, vol.~33, no.~4, pp. 917--963, 2019.

\bibitem{m4_dataset}
\BIBentryALTinterwordspacing
S.~Makridakis, E.~Spiliotis, and V.~Assimakopoulos, ``The m4 competition: 100,000 time series and 61 forecasting methods,'' \emph{International Journal of Forecasting}, vol.~36, no.~1, pp. 54--74, 2020, m4 Competition. [Online]. Available: \url{https://www.sciencedirect.com/science/article/pii/S0169207019301128}
\BIBentrySTDinterwordspacing

\bibitem{Yuqietal-2023-PatchTST}
Y.~Nie, N.~H.~Nguyen, P.~Sinthong, and J.~Kalagnanam, ``A time series is worth 64 words: Long-term forecasting with transformers,'' in \emph{International Conference on Learning Representations}, 2023.

\bibitem{burnaev2015surrogate}
E.~Burnaev and A.~Zaytsev, ``Surrogate modeling of multifidelity data for large samples,'' \emph{Journal of Communications Technology and Electronics}, vol.~60, pp. 1348--1355, 2015.

\bibitem{ober2021promises}
S.~W. Ober, C.~E. Rasmussen, and M.~van~der Wilk, ``The promises and pitfalls of deep kernel learning,'' in \emph{UAI}.\hskip 1em plus 0.5em minus 0.4em\relax PMLR, 2021, pp. 1206--1216.

\end{thebibliography}
